\documentclass{article}

\usepackage[preprint]{neurips_2025}

\usepackage[utf8]{inputenc} 
\usepackage[T1]{fontenc}    
\usepackage{hyperref}       
\usepackage{url}            
\usepackage{booktabs}       
\usepackage{amsfonts}       
\usepackage{nicefrac}       
\usepackage{microtype}      
\usepackage{xcolor}         
\usepackage{array} 
\usepackage{amsmath}
\PassOptionsToPackage{numbers}{natbib}  
\usepackage{enumitem} 
\usepackage{amsthm}
\usepackage{thmtools}
\usepackage{graphicx} 
\usepackage{multirow}
\usepackage{wrapfig}
\usepackage{todonotes}
\PassOptionsToPackage{table}{xcolor}
\usepackage[table]{xcolor}
\usepackage{colortbl}
\usepackage{comment}
\definecolor{myblue}{RGB}{31,119,180}
\newcolumntype{C}{>{\columncolor{myblue!20}}c} 
\newcolumntype{B}{>{\columncolor{myblue}\color{white}}c} 

\usepackage{makecell}

\usepackage{algorithm}
\usepackage{algorithmic}
\usepackage{amsmath}  
\usepackage{amssymb} 

\theoremstyle{definition} 
\newtheorem{theorem}{Theorem}[section]

\newtheorem{lemma}[theorem]{Lemma}
\newtheorem{corollary}[theorem]{Corollary}
\newtheorem{definition}[theorem]{Definition}
\newtheorem{assumption}[theorem]{Assumption}

\title{KCES: Training-Free Defense for Robust Graph Neural Networks via Kernel Complexity}
\author{%
  Yaning Jia, Shenyang Deng, Chiyu Ma, Yaoqing Yang, Soroush Vosoughi\thanks{Corresponding author: \texttt{soroush.vosoughi@dartmouth.edu}} \\
  Dartmouth College \\
  \texttt{\{yaning.jia.gr, shenyang.deng.gr, soroush.vosoughi\}@dartmouth.edu} \\
}

\newif\ifshowcomments
\showcommentstrue 

\ifshowcomments
  \newcommand{\yaoqing}[1]{\textcolor{teal}{[Yaoqing:\ #1]}}
  \newcommand{\shenyang}[1]{\textcolor{red}{[Shenyang:\ #1]}}
  \newcommand{\yaning}[1]{\textcolor{blue}{[Yaning:\ #1]}}
\else
  \newcommand{\yaoqing}[1]{}
  \newcommand{\shenyang}[1]{}
  \newcommand{\yaning}[1]{}
\fi

\begin{document}
\maketitle

\begin{abstract}


Graph Neural Networks (GNNs) have achieved impressive success across a wide range of graph-based tasks, yet they remain highly vulnerable to small, imperceptible perturbations and adversarial attacks. Although numerous defense methods have been proposed to address these vulnerabilities, many rely on heuristic metrics, overfit to specific attack patterns, and suffer from high computational complexity. In this paper, we propose Kernel Complexity-Based Edge Sanitization (KCES), a training-free, model-agnostic defense framework. KCES leverages Graph Kernel Complexity (GKC), a novel metric derived from the graph's Gram matrix that characterizes GNN generalization via its test error bound. Building on GKC, we define a KC score for each edge, measuring the change in GKC when the edge is removed. Edges with high KC scores, typically introduced by adversarial perturbations, are pruned to mitigate their harmful effects, thereby enhancing GNNs' robustness. KCES can also be seamlessly integrated with existing defense strategies as a plug-and-play module without requiring training. Theoretical analysis and extensive experiments demonstrate that KCES consistently enhances GNN robustness, outperforms state-of-the-art baselines, and amplifies the effectiveness of existing defenses, offering a principled and efficient solution for securing GNNs. 

\end{abstract}

\section{Introduction}
\label{intro}

Graph Neural Networks (GNNs) have emerged as powerful tools for modeling graph-structured data, achieving notable success in diverse domains such as social network analysis~\cite{perozzi2014deepwalk, grover2016node2vec}, recommendation systems~\cite{ying2018graph, he2020lightgcn}, and drug discovery~\cite{gilmer2017neural, yang2019analyzing}. Despite their widespread applicability, GNNs, similar to neural networks operating on Euclidean data~\cite{carlini2017towards, jin2020bert}, are demonstrably vulnerable to adversarial attacks~\cite{zugner2018adversarial, xu2018powerful, bojchevski2019adversarial, dai2018adversarial}. These attacks involve minor perturbations, such as targeted node manipulations~\cite{sun2020adversarial} or strategic edge modifications~\cite{geisler2021robustness}, which severely degrades model performance.

To enhance GNN robustness against adversarial attacks, various defense strategies have emerged. Heuristic purification methods like GNN-Jaccard~\cite{wu2019adversarial} and GNN-SVD~\cite{entezari2020all} refine graph structure by removing suspicious or noisy edges based on similarity or low-rank approximations. Adversarial training approaches, such as RGCN~\cite{zhu2019robust}, expose the model to perturbed graph during training to learn robust representations. Robust architecture designs like ProGNN~\cite{jin2020graph} and GNNGuard~\cite{zhang2020gnnguard} incorporate adaptive graph learning or attention-based edge filtering. However, these methods still exhibit notable limitations: \textit{\textbf{(i)}}~\textit{Reliance on heuristic metrics:} Many methods (e.g., GNN-Jaccard, GNN-SVD) use simple priors like jaccard similarity or low-rank approximations to identify suspicious edges. These heuristics lack a grounding in model generalization analysis, potentially reducing their effectiveness in identifying edges under adversarial perturbations; \textit{\textbf{(ii)}}~\textit{Overfitting to specific attacks}: Adversarial training methods like RGCN often defend against known attack patterns (e.g., Nettack~\cite{zugner2018adversarial}), which can limit their robustness to unseen attack strategies; \textit{\textbf{(iii)}}~\textit{High computational complexity}: approaches like ProGNN involve complex iterative optimization procedures, increasing runtime and limiting scalability.
 \vspace{-1em}
\begin{figure}[H]
  \centering
  \includegraphics[width=0.85\linewidth]{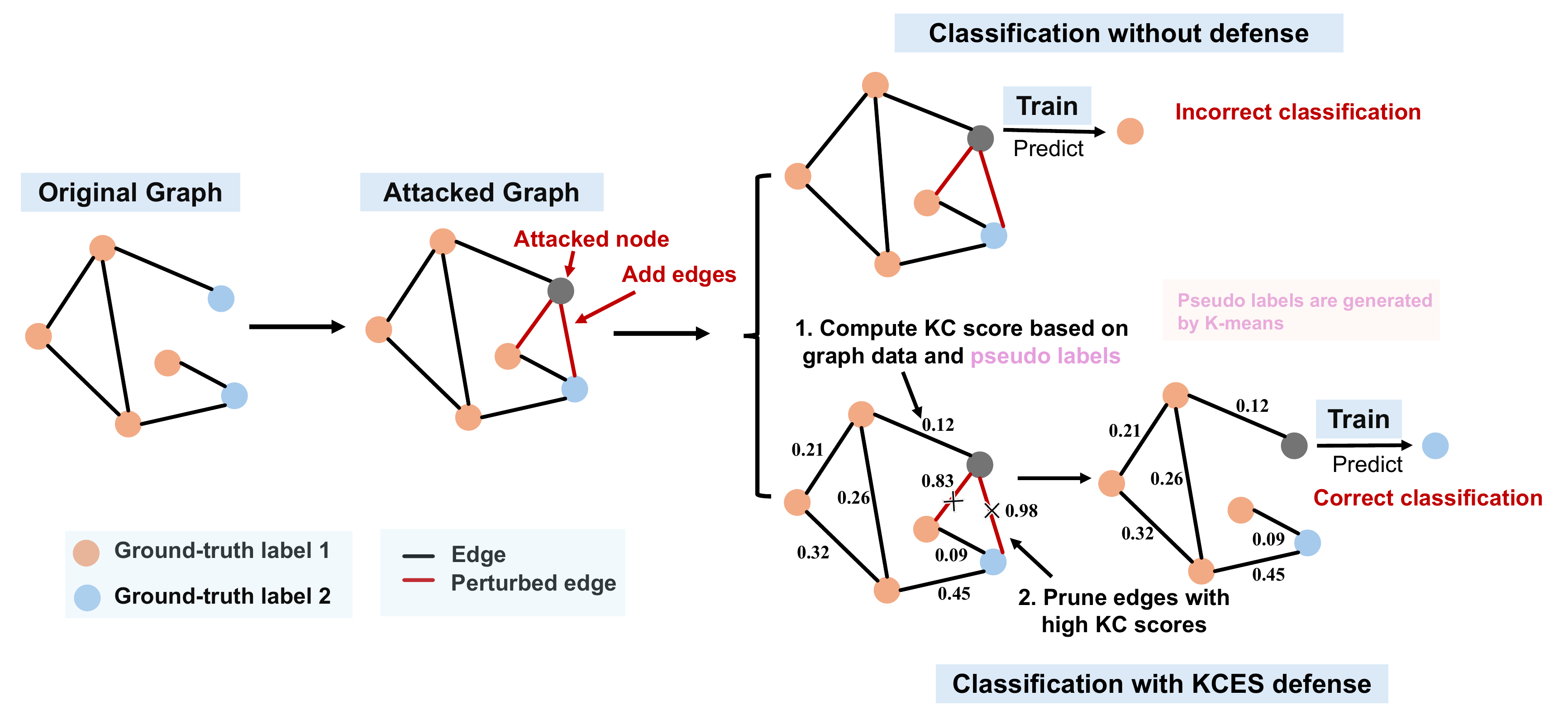}
  \caption{Defense Framework against Adversarial Perturbations with KCES.}
  \label{fig-1:kces_illustration}
\end{figure}
\vspace{-1.3em}


Motivated by the above limitations and recent progress in data-independent generalization~\cite{arora2019fine,nohyun2022data}, we introduce \textbf{Kernel Complexity-Based Edge Sanitization (KCES)}, which operates in two main steps (see Figure~\ref{fig-1:kces_illustration}). First, drawing on research in kernel models, we introduce \textbf{Graph Kernel Complexity (GKC)}, a metric that controls the test error bound of GNN. From GKC, we define the \textbf{Graph Kernel Complexity Gap Score (GKCG score)}, or simply the \textbf{KC score}, for an edge as the change in GKC upon its removal. This score indicates the corresponding change in the test error bound and serves as an effective signal for identifying adversarial edges, which typically exhibit high KC scores. Second, KCES leverages KC scores, which are computed solely from the graph data and pseudo labels, to perform targeted edge sanitization by pruning edges with high scores. This strategy, validated visually (Figure~\ref{fig-1:kces_illustration}) and experimentally (Sections~\ref{attack-score-distribution} and~\ref{edge-sanitization-parameters}), is based on the insight that edges with high KC scores are characteristic of adversarial attacks. By selectively removing these detrimental perturbations while preserving essential benign graph structure, KCES significantly improves GNN robustness.

KCES presents several key advantages over existing defense baselines. It is theoretically grounded under certain assumptions and data-driven, while avoiding reliance on heuristic assumptions about graph properties. This ensures greater flexibility and general applicability across various GNN architectures. KCES quantifies each edge’s intrinsic contribution to GNNs' generalization, which measure the probability of the edge under the adversarial perturbations, which is independent of specific attack strategies, helping prevent overfitting to known adversarial patterns and improving robustness against previously unseen attacks. Furthermore, KCES is a training-free, computationally efficient approach, typically needing only a single pre-training computation, thereby bypassing the high costs of iterative optimization methods.

Overall, our contributions are summarized as follows:
\vspace{-0.5em}
\begin{itemize}[leftmargin=*, labelindent=0pt, labelsep=1.5em, itemsep=2pt, parsep=1pt]
    \item We propose GKC, a novel data-driven, model-agnostic metric defining a key complexity factor in GNN generalization error bounds. From GKC, we derive the KC score to quantify each edge's contribution to test performance on GNNs, effectively identifying adversarially perturbed edges. 
    \item We design KCES, a universal, training-free defense that enhances graph model robustness by selectively pruning likely adversarial edges. This approach is independent of GNN model specifics or attack assumptions.
    \item Extensive experiments show that KCES outperforms popular defense baselines across diverse datasets and various GNN architectures. Moreover, it can be seamlessly integrated as a plug-and-play component to further boost the robustness of existing defense strategies.
\end{itemize}

\section{Related Works}
\label{related_works}

\paragraph{Attacks and Defenses in GNNs} Adversarial attacks in machine learning seek to degrade model performance by introducing subtle input perturbations~\cite{szegedy2013intriguing, madry2017towards, papernot2016limitations, kurakin2018adversarial, moosavi2016deepfool}. When applied to graph-structured data, these attacks compromise structural or semantic integrity through strategies such as structural perturbations~\cite{xu2019topology, chen2018fast, waniek2018hiding, zugner2018adversarial, zügner2018adversarial}, attribute manipulation~\cite{bojchevski2019adversarial, sun2020adversarial, feng2019graph}, and malicious node injection~\cite{sun2020adversarial, zou2021tdgia, tao2021single}. To counter these attacks, various strategies have been developed to enhance the robustness of machine learning models~\cite{szegedy2013intriguing, madry2017towards, papernot2016limitations, kurakin2018adversarial, moosavi2016deepfool, jia2023enhancing, jia2024aligning}. In the graph domain, several defense methods have been proposed. \textit{Robust architecture design} aims to strengthen GNNs by optimizing their structural components, as demonstrated by GNNGuard~\cite{zhang2020gnnguard} and ProGNN~\cite{jin2020graph}. \textit{Graph adversarial training} improves robustness by incorporating adversarial examples during training, as in RGCN~\cite{zhu2019robust}. \textit{Graph data purification} mitigates perturbations by denoising or correcting input data, with methods such as GNN-Jaccard~\cite{wu2019adversarial} and GNN-SVD~\cite{entezari2020all}.  However, existing defenses face several challenges, including \textit{insufficient theoretical grounding}, \textit{limited adaptability to diverse attacks}, and \textit{high computational overhead}. These limitations inspire our method’s design to overcome them.


\paragraph{Gram Matrix Applications} Gram matrices are pivotal for analyzing neural network behavior and optimization. Model-centric applications include investigating how architectures learn target functions~\cite{rahimi2007random} and exploring invariance in MLPs~\cite{tsuchida2018invariance}. From an optimization standpoint, they are crucial for understanding training dynamics, such as gradient descent in over-parameterized networks~\cite{allen2019convergence} and the interplay of learning and stability in two-layer networks~\cite{arora2019fine}. Building on this, training-free data valuation methods employing Gram matrices have quantified individual data point influence in Euclidean data~\cite{nohyun2022data}. Such approaches, however, are less explored for Graph Neural Networks (GNNs). Inspired by recent advances~\cite{arora2019fine, nohyun2022data}, we extend this concept to graph-structured data by conceptualizing a two-layer GNN as a kernel model. Then we propose KCES, a training-free and model-agnostic method to enhance GNNs' robustness against adversarial perturbations.


\section{Preliminaries}
\label{prelim}


In this section, we summarize the key notations used throughout the paper. Let $\mathbb{R}^{d}$ denote the $d$-dimensional Euclidean space, and define $[n] = \{1, 2, \ldots, n\}$. Bold symbols (e.g., $\mathbb{W}$) represent matrices, where $\mathbb{W}_{ij}$ denotes the $(i, j)$-th entry. If $\mathbb{W}$ is symmetric, then $\lambda_{\min}(\mathbb{W})$ denotes its smallest eigenvalue. Capital letters (e.g., $X$) denote vectors, and lowercase letters (e.g., $c$) denote scalars. We use $\|\cdot\|_p$ to denote the $p$-norm for vectors and the corresponding operator norm for matrices, while $\|\cdot\|_F$ refers to the Frobenius norm. A Gaussian distribution with mean $\mu$ and covariance $\Sigma$ is denoted by $\mathcal{N}(\boldsymbol{\mu}, \boldsymbol{\Sigma})$.

\subsection{Theoretical Setup for GNNs.} 
\label{gnn-definition}


We analyze a two-layer GNN as a kernel model on an undirected graph with \(N\) nodes. The graph distribution \(\mathcal{D}_G\), defined over \(\mathbb{R}^{N \times F} \times \mathbb{R}^N\), generates a single graph \(G = (\mathbb{X}, \tilde{\mathbb{A}}, \tilde{\mathbb{D}}, \mathbf{y})\), where each node feature-label pair \((\mathbb{X}_i, y_i) \in \mathbb{R}^F \times \mathbb{R}\) is drawn independently. The adjacency matrix \(\tilde{\mathbb{A}} \in \{0,1\}^{N \times N}\) is fixed, symmetric, and includes self-loops (\(\tilde{\mathbb{A}}_{ii} = 1\)), and the degree matrix \(\tilde{\mathbb{D}}\) is defined by \(\tilde{\mathbb{D}}_{ii} = \sum_j \tilde{\mathbb{A}}_{ij}\). A subset of nodes and their corresponding labels from \(G\), denoted as \(G_{\text{train}}\), is used for training. The forward propagation of a two-layer GNN for node \(i\) is given by:

\vspace{-0.5em}
\begin{equation}
\label{GCN}
f_{\text{GNN}}(\mathbb{X}_i, \tilde{\mathbb{A}}, \tilde{\mathbb{D}}) = \frac{1}{\sqrt{m}} \sum_{r=1}^m a_r \sigma\left( W_r^{\top} \left( \tilde{\mathbb{D}}^{-\frac{1}{2}} \tilde{\mathbb{A}} \tilde{\mathbb{D}}^{-\frac{1}{2}} \mathbb{X} \right)_i \right),
\end{equation}
Here, \(m\) denotes the number of hidden units; \(\mathbb{W} \in \mathbb{R}^{F \times m}\) is the first-layer weight matrix with columns \(W_r \in \mathbb{R}^F\); \(\mathbf{a} = (a_1, \dots, a_m)^{\top} \in \mathbb{R}^m\) is the second-layer weight vector with scalar entries \(a_r\); \(\sigma(\cdot)\) is the activation function (e.g., ReLU).


For the GNN \(f_{\text{GNN}}(\mathbb{X}_i, \tilde{\mathbb{A}}, \tilde{\mathbb{D}})\), we define the training error (empirical risk) on \(G_{\text{train}}\) as:
\begin{equation}
\label{eq-trainloss}
L(\mathbb{W}) = \frac{1}{2} \sum_{i=1}^{N} \left( y_i - f_{\text{GNN}}(\mathbb{X}_i, \tilde{\mathbb{A}}, \tilde{\mathbb{D}})_i \right)^2,
\end{equation}
Here, \(\mathbb{X}_i\) and \(y_i\) are sampled from \(G_{\text{train}}\). The corresponding test error (expected risk) is defined as the expectation of the training error over the graph distribution \(\mathcal{D}_G\), denoted by:
\begin{equation}
\label{eq-testloss}
L_{\mathcal{D}_G}(\mathbb{W}) = \mathbb{E}_{G \sim \mathcal{D}_G} \left[ L(\mathbb{W}) \right].
\end{equation}

\subsection{Graph Kernel Gram Matrix}
\label{gram-matrix}

We introduce the \textbf{Graph Kernel Gram Matrix} (hereafter, the \textbf{Gram matrix}), which captures pairwise node relationships based solely on graph structure and node features. This matrix builds upon classical kernel methods~\cite{xie2017diverse, arora2019fine, tsuchida2018invariance, du2018gradient}, where Gram matrices are widely used to connect model complexity with generalization performance. In our specific GNN setting, \textit{we integrate the graph aggregation operation directly into this kernel}. Consequently, for a graph \(G = (\mathbb{X}, \tilde{\mathbb{A}}, \tilde{\mathbb{D}}, \mathbf{y})\), given the normalized aggregated feature matrix $\tilde{\mathbb{X}} = \tilde{\mathbb{D}}^{-\frac{1}{2}} \tilde{\mathbb{A}} \tilde{\mathbb{D}}^{-\frac{1}{2}} \mathbb{X} \in \mathbb{R}^{N \times F}$, we define the Gram matrix $\mathbb{H}^\infty = \left[ \mathbb{H}_{ij}^\infty \right]_{i,j=1}^{N} \in \mathbb{R}^{N \times N}$, where each entry $\mathbb{H}_{ij}^\infty$ is given by:
\begin{align}
\label{gkgm}
    \mathbb{H}_{ij}^{\infty} = \frac{\tilde{\mathbb{X}}_i^{\top} \tilde{\mathbb{X}}_j \left(\pi - \arccos \left(\tilde{\mathbb{X}}_i^{\top} \tilde{\mathbb{X}}_j\right)\right)}{2 \pi},
\end{align}
Here, $\tilde{\mathbb{X}}_i \in \mathbb{R}^{1 \times F}$ and $\tilde{\mathbb{X}}_j \in \mathbb{R}^{1 \times F}$ represent the $i$-th and $j$-th rows of the matrix $\tilde{\mathbb{X}}$, respectively. The Gram matrix $\mathbb{H}^{\infty}$, computed from these aggregated features, captures the pairwise interactions between nodes. This process effectively establishes a kernel-induced feature space, which forms the basis for computing the Graph Kernel Complexity presented in the subsequent section.

\subsection{Graph Kernel Complexity}
\label{sub-GKC}
Based on the above Gram matrix, we define the \textbf{Graph Kernel Complexity (GKC)} to qualify the test error bound of GNNs, reflecting their generalization capacity. Formally, given the Gram matrix $\mathbb{H}^\infty \in \mathbb{R}^{N \times N}$ and the label vector $\mathbf{y} \in \mathbb{R}^N$, the GKC is defined as:
\begin{equation}
\label{eq-GKC}
    \text{GKC}(\mathbb{H}^{\infty})={\frac{2 \mathbf{y}^{\top} \left(\mathbb{H}^{\infty}\right)^{-1} \mathbf{y}}{N}}.
\end{equation}

Notably, \textit{$\mathbf{y}$ denotes pseudo labels generated by \textit{K-means}~\cite{macqueen1967some} in practice, rather than the ground-truth node labels.} Therefore, GKC is a \textit{model-independent} complexity metric, determined solely by intrinsic data properties. A lower GKC score reflects well-aligned and internally consistent data, which simplifies the learning process and facilitates generalization in GNNs. Theorem~\ref{thm:test-error} formalizes this relationship by showing that smaller GKC values correspond to lower test error, thereby indicating stronger generalization.

\section{GKC-based Generalization Analysis}
\label{GKCB}
This section employs the Gram matrix and GKC to establish generalization bounds for the two-layer GNN (Section~\ref{gnn-definition}). Theorems for training and test errors are presented informally for clarity, with formal statements and proofs in the Appendix.

First, Theorem~\ref{thm:train-error} characterizes the training error dynamics as following:
\begin{theorem}[\textbf{Training error dynamics}]
\label{thm:train-error}
Under Assumption E.2 in Appendix E, after $t$ gradient descent updates with step size $\eta$, the training error satisfies
\begin{equation}
\label{thm:train-error-eq}
L(\mathbb{W}_t) = \left\| \left( \mathbb{I} - \eta \mathbb{H}^\infty \right)^t \mathbf{y} \right\|_2^2 + \varepsilon,
\end{equation}
where \( \mathbb{H}^\infty \) denotes the Gram matrix, \( m \) is the hidden layer width, and \( \varepsilon = \tilde{O}(m^{-1/2}) \) represents an error term dependent on \( t \) and \( m \). The formula shows that the training error decreases with the number of training steps, and the rate of decrease is governed by the Gram matrix. A formal version of the result is provided in Appendix~E.2.1 (Theorem~E.4).
\end{theorem}

Building on training error analysis, Theorem~\ref{thm:test-error} establishes a generalization bound via GKC:
\begin{theorem}[\textbf{Test error bound}]
\label{thm:test-error}
Under Assumption E.2 in Appendix E, for sufficiently large hidden layer width $m$ and iteration count $t$, the following holds with probability at least $1 - \delta$:
\begin{equation}
\label{thm:test-error-eq}
 L_{\mathcal{D}_G}\!\left(\mathbb{W}_t\right)
 \;\le\;
 \sqrt{\text{GKC}\!\left(\mathbb{H}^{\infty}\right)}
 \;+\;
 O\!\left(\sqrt{\frac{\log \frac{N}{\lambda_0 \delta}}{N}}\right),
\end{equation}
Here, \( \text{GKC}\!\left(\mathbb{H}^{\infty}\right) \) denotes the Graph Kernel Complexity (Section~\ref{sub-GKC}), \( \delta \in (0, 1) \) is the confidence level, \( N \) is the number of nodes, and \( \lambda_0 \) is a lower bound on \( \lambda_{\min}(\mathbb{H}^{\infty}) \) (Lemma~E.3). The generalization bound on the GNN test error, as derived from the formula, is primarily influenced by the data-dependent GKC term. A smaller GKC leads to a tighter bound, indicating improved generalization under standard GNN training regimes. The formal statement and proof are presented in Appendix~E.2.1 (Theorem~E.5).
\end{theorem}

\section{Kernel Complexity-Based Edge Sanitization for Robustness}
Motivated by the established connection between GKC and the test error of GNNs, we propose \textbf{Kernel Complexity-Based Edge Sanitization} (\textbf{KCES}), a training-free and model-agnostic defense framework that evaluates the importance of graph edges and strategically prunes them to improve GNN robustness against adversarial perturbations. This framework comprises two key components: \textbf{\textit{(i) Edge KC Score Estimation}} and \textbf{\textit{(ii) KC-Based Edge Sanitization}}. The first component assigns a kernel-based score to each edge, quantifying its structural importance. The second component utilizes these scores to selectively remove potentially harmful edges, thereby mitigating the impact of adversarial attacks.

\subsection{Edge KC Score Estimation} 
\label{KCscoredef}
To quantify the importance of each edge, we introduce the \textbf{Graph Kernel Complexity Gap Score}, hereafter referred to as the \textbf{KC score}. Specifically, for an edge $e_{ij}$ connecting the $i$-th and $j$-th nodes in the graph, the KC score is defined as:
\begin{equation}
\label{kc_def}
  KC(i, j) = \left| \text{GKC}\left(\mathbb{H}^{\infty}\right) - \text{GKC}\left(\mathbb{H}^{\infty}_{-(i,j)}\right) \right|,
\end{equation}
Here, $\mathbb{H}^{\infty}$ denotes the Gram matrix of the original graph, and $\mathbb{H}^{\infty}_{-(i,j)}$ corresponds to the Gram matrix of the graph $G_{-(i,j)}$, obtained by removing edge $e_{ij}$. Both matrices are computed using pseudo labels $\mathbf{y}$ generated by \textit{K-means}; the details of generating $\mathbf{y}$ refers to Appendix B. The KC score \(KC(i,j)\) quantifies the change in GKC resulting from the removal of edge \(e_{ij}\). This score governs the upper bound of the interval within which the test error may vary, as stated in Corollary~\ref{edge_theo} (For a formal version and proof of the corollary, see Appendix E.2.2 and E.4.3).

\begin{corollary}[Edge‑specific test error bound]
\label{edge_theo}
Let the graph $G_{-(i,j)}$ be obtained by deleting edge $e_{ij}$, and let $\mathbb{W}_{t}$ denote the GNN parameters after $t$ gradient descent updates on the modified graph. Under the same assumptions as Theorem~\ref{thm:test-error}, and for sufficiently large hidden width $m$ and iteration count $t$, the following holds with probability at least $1 - \delta$:
\begin{equation}
    L_{\mathcal{D}_{G_{-(i,j)}}}\left(\mathbb{W}_{t}\right) \le \sqrt{\text{GKC}\left(\mathbb{H}^{\infty}\right)} + \sqrt{KC(i,j)} + O\left(\sqrt{\frac{\log \frac{N}{\lambda_0 \delta}}{N}}\right),
\end{equation}
where $N$ denotes the number of nodes; $\lambda_0$ is the lower bound on $\lambda_{\min}(\mathbb{H}^{\infty})$, as specified in Assumption E.2; and $\mathcal{D}_{G_{-(i,j)}}$ represents the data distribution for graph $G_{-(i,j)}$, which are generated by taking graphs from $\mathcal{D}_G$ and removing edge $e_{ij}$. Corollary~\ref{edge_theo} establishes a theoretical link between the KC score of an edge \( e_{ij} \) and the GNN's expected test error on the graph \( G_{-(i,j)} \) (i.e., after removing \( e_{ij} \)), thereby quantifying the edge's contribution to the overall GKC. Adversarially perturbed edges, being theoretically detrimental to GNN generalization, are expected to exhibit higher KC scores. This is empirically supported by the results in Section~\ref{attack-score-distribution}, which shows attacks increase the prevalence of such high-impact edges. Thus, the KC score serves as an effective indicator of these adversarial structural perturbations, guiding the strategic removal of detrimental edges to enhance GNN robustness.
\end{corollary}

\subsection{KC-based Edge Sanitization}
Building on estimated KC scores, we can use them to sanitize graphs. Since edges with high scores tend to negatively impact model performance, we can remove edges with higher KC scores to ensure GNN performance, enhancing its robustness. Specific \textbf{KCES} refer to the following Algorithm~\ref{alg:kces}.

\begin{algorithm}[htbp] 
 \caption{KCES: Kernel Complexity-Based Edge Sanitization} 
 \label{alg:kces} 
 \textbf{Input:} Graph $G = (V, E)$; features $\mathbb{X}$; pruning ratio $\alpha \in [0,1]$  \\
 \textbf{Output:} Sanitized graph $G' = (V, E')$ 
 \begin{algorithmic}[1] 
  \STATE $\mathbf{\hat{y}} \leftarrow \operatorname{K-means}(\tilde{\mathbb{A}}, \mathbb{X})$ \hfill \small{\textit{// Generate pseudo labels for nodes via \textit{K-means}}} 
 \STATE $\mathbb{H}^{\infty} \leftarrow \operatorname{GKGM}(G, \mathbb{X})$ 
    \hfill \small{\textit{// Compute original Gram matrix, see Eq.~\ref{gkgm}}} 
   \vspace{0.2em}
 \STATE $\text{GKC}(\mathbb{H}^{\infty}) \leftarrow \frac{2 \mathbf{\hat y}^\top (\mathbb{H}^{\infty})^{-1} \mathbf{\hat y}}{N}$ 
    \hfill \small{\textit{// Compute original GKC, see Eq.~\ref{eq-GKC}}} 
   \vspace{0.2em}
 \STATE $\operatorname{KC} \leftarrow [\,]$ 
    \hfill \small{\textit{// Initialize a KC list for edge scores}}
    \vspace{0.3em}
 \FOR{each edge $e_{ij} \in E$} 
    \STATE $G_{-(i,j)} \leftarrow G \setminus \{e_{ij}\}$ 
        \hfill \small{\textit{// Create a modified graph excluding edge $e_{ij}$}} 
    \vspace{0.2em}
    \STATE $\mathbb{H}^{\infty}_{-(i,j)} \leftarrow \operatorname{GKGM}(G_{-(i,j)}, \mathbb{X})$ 
        \hfill \small{\textit{// Compute Gram matrix for the modified graph, see Eq.~\ref{gkgm}}} 
    \vspace{0.2em}
    \STATE $\text{GKC}(\mathbb{H}^{\infty}_{-(i,j)})\leftarrow \frac{2 \mathbf{\hat y}^\top (\mathbb{H}^{\infty}_{-(i,j)})^{-1} \mathbf{\hat y}}{N}$ 
        \hfill \small{\textit{// Compute GKC for the modified graph, see Eq.~\ref{eq-GKC}}} 
    \vspace{0.2em}
    \STATE $\operatorname{KC}[e_{ij}] \leftarrow \left| \text{GKC}(\mathbb{H}^{\infty} - \text{GKC}(\mathbb{H}^{\infty}_{-(i,j)}) \right|$ 
        \hfill \small{\textit{// Compute KC score for $e_{ij}$, see Eq.~\ref{kc_def}}}
  \ENDFOR 
  \vspace{0.3em}
 \STATE $(e_{(1)}, e_{(2)}, \dots, e_{(|E|)}) \leftarrow \operatorname{SortEdgesByScore}(\operatorname{KC}, \text{descending})$ 
    \hfill \small{\textit{// Sort edges by KC score}}
  \vspace{0.2em}
 \STATE $k \leftarrow \lceil \alpha \cdot |E| \rceil$ 
    \hfill \small{\textit{// Number of edges to remove}} 
    \vspace{0.2em}
 \STATE $E' \leftarrow E \setminus \{e_{(1)}, \dots, e_{(k)}\}$ 
    \hfill \small{\textit{// Prune the top-$k$ KC score edges}} 
    \vspace{0.2em}
 \STATE \textbf{Return} $G' = (V, E')$ \hfill \small{// \textit{Return Graph with pruned edge set}}
 \end{algorithmic} 
\end{algorithm}

\vspace{-0.6em}
\section{Experiments}

This section details four comprehensive experiments designed to validate our theoretical analyses and assess the KCES method's effectiveness against adversarial attacks. \textbf{Experiment~1} examines KC score distributions in clean and adversarial graphs, explaining the effect of adversarial perturbations on KC score. \textbf{Experiment~2} shows connection between pruning strategies and model generalization, empirically demonstrating edge KC score is highly linked with GNN performance. \textbf{Experiment~3} assesses KCES's robustness across diverse adversarial settings, demonstrating state-of-the-art results. Finally, \textbf{Experiment~4} showcases KCES's plug-and-play compatibility with existing defense baselines, highlighting its capacity to further improve their robustness.


\subsection{Overall Setup}
\label{setup}
\paragraph{Dataset} We evaluate the robustness and scalability of the proposed KCES method using five benchmark datasets spaning three distinct scales: \textbf{(1)} small-scale graphs (fewer than 10,000 edges): \textit{Cora}\cite{sen2008collective}, \textit{Citeseer} \cite{sen2008collective}, and \textit{Polblogs\cite{adamic2005political}}; \textbf{(2)} medium-scale graph (10,000-100,000 edges): \textit{Pubmed} \cite{yang2016revisiting}; and \textbf{(3)} large-scale graph (exceeding 100,000 edges): \textit{Flickr}\cite{liu2009social}.  Detailed dataset statistics are provided in Appendix A.
\vspace{-0.5em}
\paragraph{Attack strategies} We utilize five representative attack methods to evaluate robustness, categorized into three types. \textit{These strategies involve perturbing the graph data before GNN training, with defense methods subsequently applied to mitigate their impact.} The categories are: \textbf{\textit{(i) Non-targeted Attack}}, which degrade overall GNN performance by perturbing the entire graph structure, including \textit{Metattack}~\cite{zügner2018adversarial}, \textit{MINMAX}~\cite{xu2019topology}, and \textit{DICE}~\cite{waniek2018hiding}; \textbf{\textit{(ii) Targeted Attack}}, which focus on misleading GNN predictions for specific nodes, for which we employ the widely-used \textit{Nettack}~\cite{zugner2018adversarial}; and \textbf{\textit{(iii)~Random Attack}}, which simulate noise by randomly adding or removing edges at, termed \textit{Random}.
\vspace{-0.5em}
\paragraph{Baseline methods}
We evaluate the effectiveness of KCES by comparing it against several state-of-the-art defense methods from two complementary perspectives. First, we consider a set of widely adopted and robust GNN architectures, including GCN~\cite{xu2019topology}, GAT~\cite{velivckovic2017graph}, and RGCN~\cite{zhu2019robust}. Second, we benchmark KCES against established defense strategies specifically designed to improve the robustness of GNN in adversarial settings. These include ProGNN~\cite{jin2020graph}, GNN-Jaccard~\cite{wu2019adversarial}, GNN-SVD~\cite{entezari2020all}, and GNNGuard~\cite{zhang2020gnnguard}, all of which represent mainstream approaches for defending against both non-targeted and targeted attacks. 
\paragraph{Implementation details} We use test accuracy for node classification as the primary evaluation metric in the following experiments. Further implementation details are provided in Appendix~C.
 
\subsection{Impact of Adversarial Perturbations on Kernel Complexity Gap Scores}
\label{attack-score-distribution}

This experiment investigates the \textbf{\textit{relationship between KC score distribution and adversarial perturbations}}. Based on Corollary~\ref{edge_theo}, adversarially perturbed edges are usually detrimental to GNN generalization, which may increase GKC theoretically, resulting in a larger test error bound (expressed as large KC scores). To verify this, we compare KC score distributions across three settings: . To verify this, we compare KC score distributions across three settings: \textbf{(1)} the original clean graph, \textbf{(2)} an adversarially perturbed graph (using \textit{Metattack}), and \textbf{(3)} the perturbed graph after pruning its high-KC score edges. For a fair comparison, an equal number of edges are randomly sampled from each graph type, and their KC scores are normalized to the $[0,1]$ range. Kernel Density Estimation (KDE)~\cite{parzen1962estimation} is used to visualize these smoothed distributions. Experiments are conducted on \textit{Cora} and \textit{Pubmed} datasets (representing varying graph scales) using a GCN. We sample 1,000 edges from \textit{Cora} and 10,000 from \textit{Pubmed}, proportional to their sizes. Pruning ratios are set to 0.25 for \textit{Cora} and 0.75 for \textit{Pubmed}, based on findings in Section~\ref{edge-sanitization-parameters}.
\begin{figure}[htbp]
  \centering
  \includegraphics[width=0.95\linewidth]{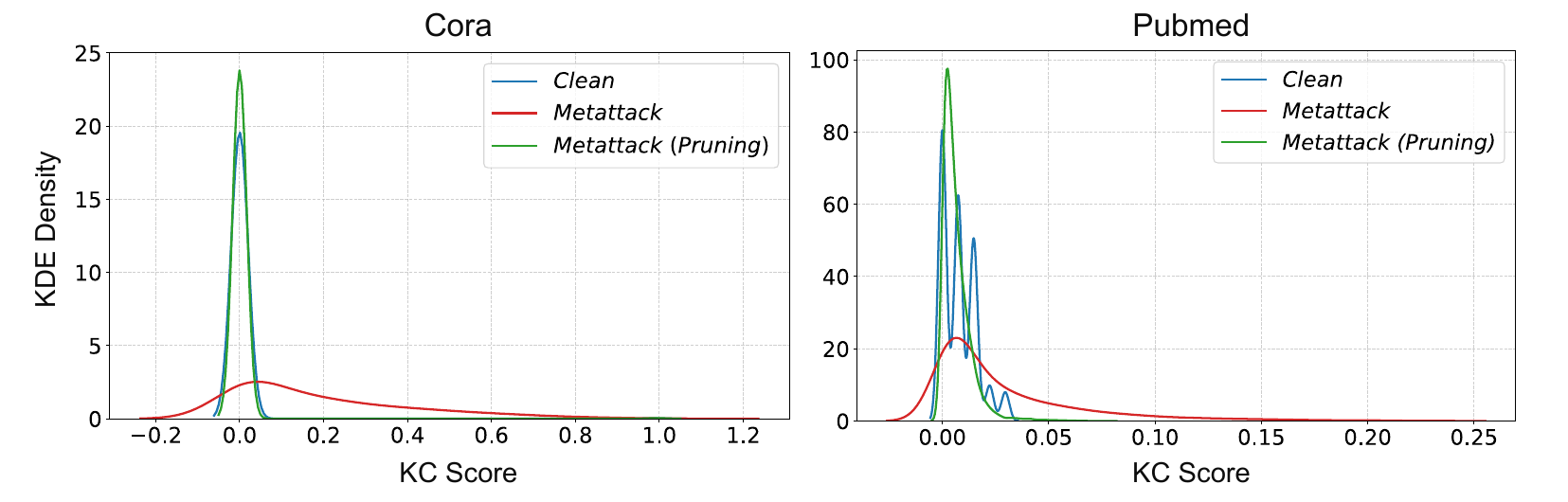}
  \vspace{-0.5em}
  \caption{Visualization of KC Score Distributions Across \textit{Clean}, \textit{Metattack}, and \textit{Pruned Graphs}.}
  \label{fig:score-distribution}
\end{figure}

The results, depicted in Figure~\ref{fig:score-distribution}, demonstrate that adversarial perturbations markedly increase edge KC scores, thereby empirically supporting Corollary~\ref{edge_theo}. Specifically, the KC score distribution for clean graphs is highly concentrated around zero, reflecting a low prevalence of edges with high KC scores. In contrast, \textit{Metattack} induces a significant shift in the distribution towards higher KC values and introduces a long tail. This suggests an increase in structurally redundant or uninformative edges, which adversely affect GNN performance as edges with high KC scores contribute less to effective learning. The proposed pruning strategy effectively counteracts this shift by removing these high-KC score edges, thereby restoring the distribution to a form that closely resembles that of the clean graph. These findings align with our theoretical analysis in Section~\ref{GKCB}, which posits that edges with high KC scores are significant contributors to graph complexity and thus detrimental. The subsequent section will further explore the impact of edges with high and low KC scores on GNN performance.

\subsection{Impact of Edge Sanitization on Model Performance}
\label{edge-sanitization-parameters}
Section~\ref{attack-score-distribution} and Corollary~\ref{edge_theo} establish a link between each edge's influence on GNN performance and its corresponding test error, as quantified by the KC score. Edges with lower KC scores (Low-KC), such as unperturbed ones, are typically associated with lower test error and improved GNN performance. In contrast, edges with higher KC scores (High-KC), often resulting from adversarial perturbations, tend to increase test error and degrade performance. To empirically validate this insight, we prune edges based on their KC scores and assess the impact on model performance. Three distinct pruning strategies are evaluated: \textbf{\textit{(i) High-KC Pruning}}, removing edges with the highest KC scores; \textbf{\textit{(ii) Low-KC Pruning}}, removing edges with the lowest KC scores; and \textbf{\textit{(iii) Random Pruning}}, removing edges randomly as a baseline. Each strategy is tested with pruning ratios ranging from 0.05 to 0.95, in 0.05 increments. The primary evaluations utilize adversarially perturbed versions of the \textit{Cora} and \textit{Pubmed} datasets (attacked via \textit{Metattack}), with a GCN architecture employed consistently. Results from corresponding experiments on clean graphs are detailed in Appendix D.1.

The result presented in Figure~\ref{fig:ES-adversarial} highlights the significant positive contribution of Low-KC edges to GNN performance, in contrast to the demonstrably negative impact of High-KC edges. This finding, reinforced by analogous results on clean graphs (Appendix D.1), validates High-KC Pruning as an effective method for removing detrimental edges in adversarial scenarios. It thus represents a robust strategy for improving GNN generalization and bolstering robustness against adversarial perturbations. Specifically, under \textit{Metattack}, High-KC Pruning consistently sustains or improves GNN performance. In stark contrast, Low-KC Pruning generally leads to a decline in accuracy, often performing worse than Random Pruning. Notably, under adversarial conditions, both Low-KC and Random Pruning can yield performance improvements at high pruning ratios (beyond 0.5). This phenomenon is likely attributable to the extensive removal of adversarial edges, where the benefit of eliminating numerous detrimental perturbations outweighs the cost of losing some benign edges. High-KC Pruning, however, does not exhibit this phenomenon, suggesting its superior precision in targeting harmful edges while preserving essential ones. Overall, the efficacy of these pruning strategies follows the order: \textbf{\textit{High-KC Pruning > Random Pruning > Low-KC Pruning}}. These empirical results align well with our theoretical analyses in Section~\ref{GKCB}.

\begin{figure}[htbp]
  \label{fig:ES-adversarial}
  \centering
  \includegraphics[width=0.9\linewidth]{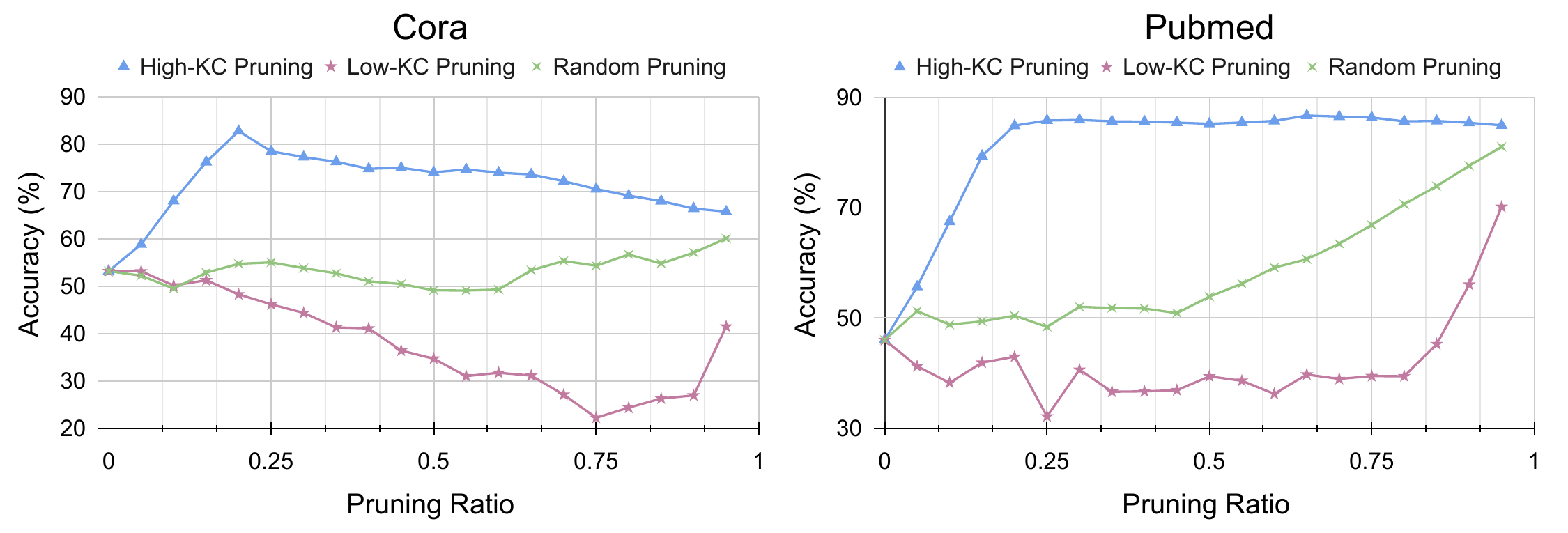}
  \vspace{-0.5em}
  \caption{Comparing Edge Pruning Strategies via KC Scores in Adversarial Settings.}
\end{figure}
\vspace{-1em}

\subsection{Defense against adversarial attacks}
In this section, we evaluate that KCES can improve GNNs' robustness against various adversarial attacks. The experimental setup, including datasets, attack methods, and defense baselines, follows the configuration described in Section~\ref{setup}. For non-targeted and random attacks, we allocate 25\% of the graph’s edges as the perturbation budget. For the targeted attack (\textit{Nettack}), we select nodes with degrees greater than 10 and perturb their connected edges. We implement KCES on GCN and compare its robustness against other baseline methods. All results are reported as percentages, and \text{N/A} indicates that the method is not applicable to the corresponding attack setting.

Table~\ref{tab-1:robustness} reports the defense performance of various methods across five datasets under random, untargeted, and targeted attacks. Overall, our proposed method, KCES, consistently outperforms baselines and achieves state-of-the-art robustness across all attack settings. Notably, KCES enables the model to match, and occasionally exceed its performance on clean setting, even under adversarial perturbations. On the \textit{Flickr} dataset, we observe that several baselines perform better under adversarial conditions than on clean data. This is likely due to the presence of numerous noisy or irrelevant edges in the large-scale graph, which hinder model performance. This pattern is consistent with prior findings~\citep{dai2022towards, dong2023towards}, suggesting that pruning-based defense methods, including ours, can reduce the impact of such noise, effectively acting as regularization and enhancing both robustness and generalization. Overall, these results confirm the effectiveness of KCES over baselines.

\begin{table}[t]
\caption{Defense performance (Accuracy $\pm$ Std) under targeted and untargeted adversarial attacks.}
\label{tab-1:robustness}
\centering
\resizebox{\textwidth}{!}{%
\begin{tabular}{llcccccccc}
\toprule
\textbf{Dataset} & \textbf{Attack} & \textbf{GCN} & \textbf{GAT} & \textbf{RGCN} & \textbf{ProGNN} & \textbf{GNN-SVD} & \textbf{GNN-Jaccard} & \textbf{GNNGuard} & \textbf{KCES (Ours)} \\
\midrule
\multirow{6}{*}{\makecell[l]{\textit{Polblogs}\\ (\textit{small})}}  
 & \textit{Clean} & $95.71 \pm 0.79$ & $95.40 \pm 0.41$ & $95.29 \pm 0.52$ & $95.60 \pm 0.46$ & $94.68 \pm 0.88$ & \text{N/A} & \text{N/A} & $\textbf{96.01} \pm \textbf{0.18}$ \\
 & \textit{Random} & $81.29 \pm 0.61$ & $85.01 \pm 0.02$ & $82.23 \pm 0.76$ & $86.50 \pm 0.12$ & $88.34 \pm 0.83$ & \text{N/A} & \text{N/A} & $\textbf{89.46} \pm \textbf{0.77}$ \\
 & \textit{Nettack} & $92.59 \pm 0.73$ & $85.79 \pm 0.23$ & $93.15 \pm 0.70$ & $95.20 \pm 0.59$ & $95.37 \pm 0.57$ & \text{N/A} & \text{N/A} & $\textbf{96.48} \pm \textbf{0.82}$ \\
 & \textit{DICE} & $71.47 \pm 0.98$ & $77.10 \pm 0.87$ & $71.16 \pm 0.36$ & $74.74 \pm 0.18$ & $76.48 \pm 0.73$ & \text{N/A} & \text{N/A} & $\textbf{93.04} \pm \textbf{0.06}$ \\
 & \textit{MINMAX} & $70.14 \pm 0.54$ & $68.04 \pm 0.89$ & $82.10 \pm 0.42$ & $60.73 \pm 0.37$ & $62.47 \pm 0.72$ & \text{N/A} & \text{N/A} & $\textbf{94.17} \pm \textbf{0.47}$ \\
 & \textit{Metattack} & $63.09 \pm 0.12$ & $63.60 \pm 0.20$ & $62.67 \pm 0.13$ & $65.23 \pm 0.24$ & $81.28 \pm 0.58$ & \text{N/A} & \text{N/A} & $\textbf{82.72} \pm \textbf{0.32}$ \\
\midrule
\multirow{6}{*}{\makecell[l]{\textit{Cora}\\ \textit{(small)}}}  %
 & \textit{Clean} & $83.65 \pm 0.57$ & $83.90 \pm 0.15$ & $82.89 \pm 0.99$ & $\textbf{85.16} \pm \textbf{0.77}$ & $78.16 \pm 0.07$ & $82.69 \pm 0.74$ & $78.87 \pm 0.77$ & $84.04 \pm 0.42$ \\
 & \textit{Random} & $77.57 \pm 0.38$ & $79.23 \pm 0.75$ & $74.55 \pm 0.13$ & $79.58 \pm 0.32$ & $78.87 \pm 0.36$ & $77.36 \pm 0.93$ & $77.11 \pm 0.57$ & $\textbf{79.67} \pm \textbf{0.69}$ \\
 & \textit{Nettack} & $57.83 \pm 0.78$ & $58.23 \pm 0.08$ & $59.04 \pm 0.51$ & $69.67 \pm 0.68$ & $74.70 \pm 1.21$ & $76.69 \pm 0.64$ & $62.65 \pm 0.03$ & $\textbf{82.24} \pm \textbf{0.43}$ \\
 & \textit{DICE} & $76.16 \pm 0.69$ & $77.06 \pm 0.31$ & $73.59 \pm 0.91$ & $75.95 \pm 0.40$ & $72.68 \pm 0.72$ & $77.11 \pm 0.22$ & $75.50 \pm 0.26$ & $\textbf{82.59} \pm \textbf{0.09}$ \\
 & \textit{MINMAX} & $60.66 \pm 0.23$ & $61.26 \pm 0.73$ & $59.80 \pm 0.75$ & $64.43 \pm 0.23$ & $59.45 \pm 0.10$ & $72.23 \pm 0.69$ & $70.57 \pm 0.95$ & $\textbf{78.42} \pm \textbf{0.47}$ \\
 & \textit{Metattack} & $53.12 \pm 0.83$ & $58.35 \pm 0.22$ & $51.35 \pm 0.72$ & $63.37 \pm 0.16$ & $61.92 \pm 0.28$ & $75.28 \pm 0.19$ & $70.77 \pm 0.97$ & $\textbf{82.99} \pm \textbf{0.08}$ \\
\midrule
\multirow{6}{*}{\makecell[l]{\textit{Citeseer}\\ (\textit{small})}}  
 & \textit{Clean} & $72.51 \pm 0.61$ & $72.69 \pm 0.55$ & $71.97 \pm 0.60$ & $71.74 \pm 0.59$ & $69.60 \pm 0.56$ & $72.98 \pm 0.06$ & $71.03 \pm 0.19$ & $\textbf{73.02} \pm \textbf{0.93}$ \\
 & \textit{Random} & $70.38 \pm 0.83$ & $69.31 \pm 0.95$ & $67.06 \pm 0.26$ & $72.36 \pm 0.19$ & $67.59 \pm 0.15$ & $71.21 \pm 0.03$ & $72.73 \pm 0.60$ & $\textbf{72.81} \pm \textbf{0.62}$ \\
 & \textit{Nettack} & $52.38 \pm 0.94$ & $59.19 \pm 0.46$ & $49.21 \pm 0.97$ & $72.23 \pm 1.04$ & $74.60 \pm 0.51$ & $72.14 \pm 0.69$ & $72.95 \pm 0.58$ & $\textbf{76.14} \pm \textbf{0.60}$ \\
 & \textit{DICE} & $67.71 \pm 0.72$ & $66.60 \pm 0.43$ & $66.17 \pm 0.85$ & $72.15 \pm 0.61$ & $67.35 \pm 0.03$ & $71.14 \pm 0.19$ & $69.01 \pm 0.34$ & $\textbf{72.45} \pm \textbf{0.84}$ \\
 & \textit{MINMAX} & $66.29 \pm 0.45$ & $67.54 \pm 0.43$ & $61.02 \pm 0.44$ & $69.90 \pm 0.80$ & $64.57 \pm 0.86$ & $71.20 \pm 0.02$ & $68.60 \pm 0.86$ & $\textbf{72.80} \pm \textbf{0.36}$ \\
 & \textit{Metattack} & $57.64 \pm 0.59$ & $61.20 \pm 0.66$ & $56.81 \pm 0.75$ & $66.33 \pm 0.46$ & $66.29 \pm 0.39$ & $70.14 \pm 0.53$ & $64.75 \pm 0.60$ & $\textbf{71.86} \pm \textbf{0.87}$ \\
\midrule
\multirow{6}{*}{\makecell[l]{\textit{Pubmed}\\ (\textit{medium})}} 
 & \textit{Clean} & $85.72 \pm 0.05$ & $84.89 \pm 0.84$ & $84.72 \pm 0.06$ & $85.19 \pm 0.43$ & $84.53 \pm 0.86$ & $86.19 \pm 0.97$ & $84.49 \pm 0.79$ & $\textbf{86.17} \pm \textbf{0.52}$ \\
 & \textit{Random} & $84.11 \pm 0.28$ & $81.02 \pm 0.72$ & $83.75 \pm 0.10$ & $84.28 \pm 0.41$ & $82.61 \pm 0.63$ & $84.27 \pm 0.64$ & $83.87 \pm 0.03$ & $\textbf{85.83} \pm \textbf{0.82}$ \\
 & \textit{Nettack} & $66.67 \pm 0.36$ & $76.73 \pm 0.88$ & $72.58 \pm 0.09$ & $72.60 \pm 0.14$ & $80.10 \pm 0.45$ & $85.48 \pm 0.12$ & $83.33 \pm 0.65$ & $\textbf{86.24} \pm \textbf{0.09}$ \\
 & \textit{DICE} & $81.68 \pm 0.46$ & $76.93 \pm 0.19$ & $81.44 \pm 0.56$ & $80.73 \pm 0.30$ & $80.39 \pm 0.12$ & $82.93 \pm 0.70$ & $82.28 \pm 0.70$ & $\textbf{85.74} \pm \textbf{0.39}$ \\
 & \textit{MINMAX} & $55.67 \pm 0.46$ & $60.01 \pm 0.97$ & $54.64 \pm 0.24$ & $69.29 \pm 0.37$ & $80.50 \pm 0.06$ & $84.51 \pm 0.83$ & $81.69 \pm 0.57$ & $\textbf{85.64} \pm \textbf{0.97}$ \\
 & \textit{Metattack} & $46.08 \pm 0.39$ & $49.72 \pm 0.37$ & $45.99 \pm 0.92$ & $72.08 \pm 0.51$ & $82.75 \pm 0.25$ & $84.22 \pm 0.39$ & $83.37 \pm 0.89$ & $\textbf{86.45} \pm \textbf{0.45}$ \\
\midrule

\multirow{6}{*}{\makecell[l]{\textit{Flickr}\\ (\textit{large})}}  
 & \textit{Clean} & $56.24 \pm 0.81$ & $47.83 \pm 0.25$ & $39.62 \pm 0.73$ & $54.68 \pm 0.68$ & $60.46 \pm 0.55$ & $74.04 \pm 0.38$ & $74.31 \pm 0.74$ & $\textbf{76.25} \pm \textbf{0.93}$ \\
 & \textit{Random} & $62.82 \pm 0.62$ & $60.80 \pm 0.42$ & $62.44 \pm 0.16$ & $65.43 \pm 0.90$ & $76.54 \pm 0.41$ & $74.83 \pm 0.96$ & $74.85 \pm 0.55$ & $\textbf{76.74} \pm \textbf{0.65}$ \\
 & \textit{Nettack} & $38.82 \pm 0.55$ & $43.12 \pm 0.39$ & $59.87 \pm 0.36$ & $70.31 \pm 0.74$ & $58.70 \pm 0.45$ & $72.58 \pm 0.48$ & $\textbf{74.51} \pm \textbf{0.33}$ & $73.47 \pm 0.97$ \\
 & \textit{DICE} & $51.71 \pm 0.32$ & $48.11 \pm 0.11$ & $47.34 \pm 0.39$ & $71.23 \pm 0.17$ & $73.38 \pm 0.93$ & $73.35 \pm 0.02$ & $73.59 \pm 0.24$ & $\textbf{73.69} \pm \textbf{0.67}$ \\
 & \textit{MINMAX} & $14.57 \pm 0.97$ & $11.71 \pm 0.26$ & $27.13 \pm 0.02$ & $19.68 \pm 0.58$ & $39.17 \pm 0.78$ & $74.82 \pm 0.33$ & $75.04 \pm 0.08$ & $\textbf{76.86} \pm \textbf{0.90}$ \\
 & \textit{Metattack} & $36.93 \pm 0.86$ & $37.29 \pm 0.06$ & $31.72 \pm 0.21$ & $65.05 \pm 0.82$ & $59.08 \pm 0.36$ & $74.85 \pm 0.95$ & $75.05 \pm 0.75$ & $\textbf{76.63} \pm \textbf{0.82}$ \\
\bottomrule
\end{tabular}%
}
\end{table}

\vspace{-1em}
\subsection{Effect of \textbf{KCES} as a Plug-and-Play Module for Enhancing Existing Defenses}

KCES is a training-free and data-independent method, making it highly adaptable as a plug-and-play module to enhance the robustness of existing models. To demonstrate its versatility and effectiveness, we apply KCES to several baseline methods listed in Table~\ref{tab-1:robustness}, including robust GNN architectures (GAT and RGCN) as well as defense algorithms (ProGNN, GNN-SVD, GNN-Jaccard, and GNNGuard).

\begin{figure}[htbp]
  \centering
  \includegraphics[width=1.0\linewidth]{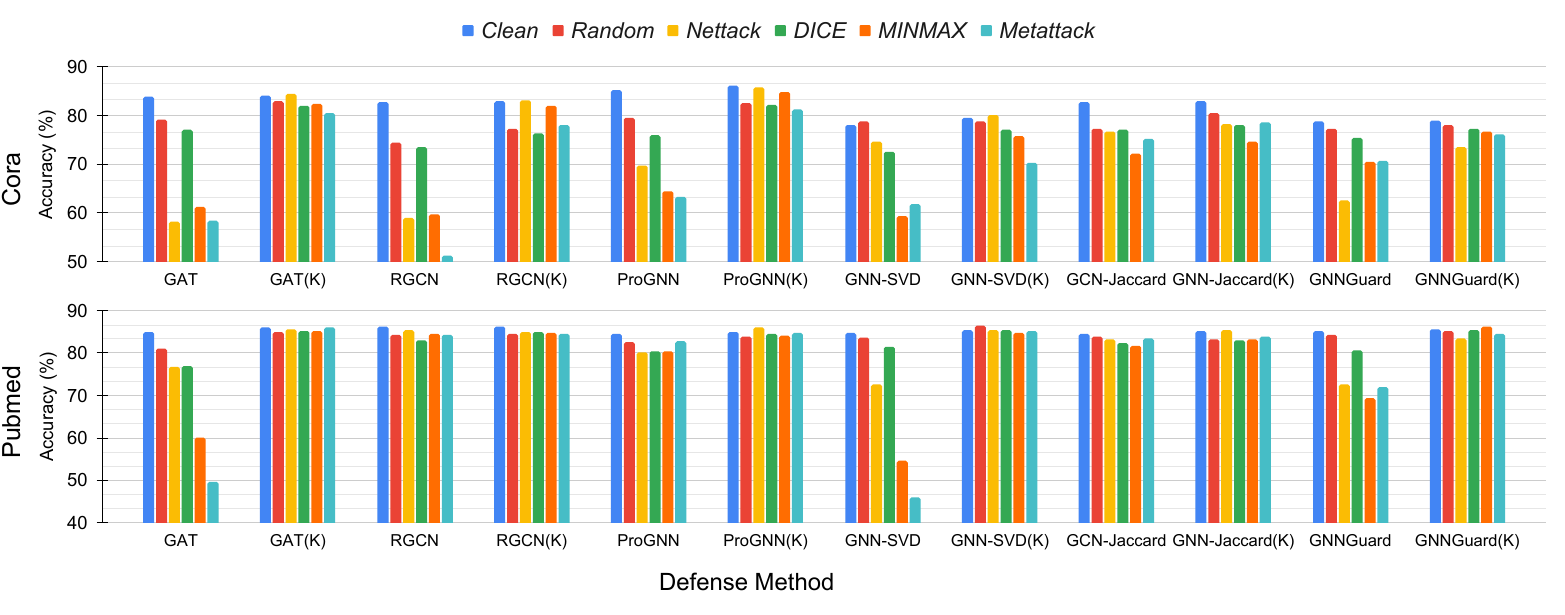}
  \vspace{-2em}
  \caption{Robustness Enhancement Against Adversarial Attacks via \textbf{KCES} Integration.}
  \label{fig-1:plug and play}
\end{figure}

The results, presented in Figure~\ref{fig-1:plug and play}, demonstrate that integrating KCES consistently improves the performance of all baseline methods across diverse adversarial attacks, underscoring its effectiveness as a plug-and-play module. Notably, KCES substantially enhances the robustness of both GNN architectures and defense algorithms, with all KCES-integrated baselines achieving over 70\% accuracy on \textit{Cora} and over 80\% on \textit{Pubmed} under a wide range of adversarial attacks. These findings indicate that KCES does not conflict with existing defense methods and can be seamlessly integrated to further strengthen their robustness against attacks.
\section{Conclusion and Limitation}
In this work, we proposed KCES, a training-free and model-agnostic framework for defending GNNs against adversarial attacks. At its core is the GKC, a novel metric derived from the graph’s Gram matrix to quantify generalization. Building on GKC, we introduced the KC score, which measures the impact of each edge on generalization by evaluating the change in GKC after edge removal. Our theoretical analysis established a direct connection between KC scores and test error, enabling principled edge sanitization. Experiments across diverse datasets and attack settings showed that KCES consistently outperforms baselines and can serve as a plug-and-play component for improving existing defenses. Despite its demonstrated effectiveness, KCES possesses certain limitations. Firstly, its primary focus is on structural perturbations, meaning it does not inherently address adversarial attacks targeting node features. Secondly, as an edge sanitization technique, KCES is not directly applicable to graph-based tasks that are intrinsically edge-centric. Thirdly, while its implementation is model-agnostic, the theoretical underpinnings of KCES rely on assumptions specific to GNNs, which may limit its direct generalization to architectures outside the GNN family. Overcoming these limitations could lead to more robust and versatile iterations of KCES, thereby further advancing GNN security in complex, real-world applications.


\bibliographystyle{unsrt}   
\bibliography{ref}
\newpage

\appendix
\section{Dataset Statistics and Splits}
\label{sec:dataset}
We use five widely adopted benchmark datasets in our experiments: \textit{Polblogs}, \textit{Cora}, \textit{Citeseer}, \textit{Pubmed}, and \textit{Flickr}. Their basic statistics, including the number of nodes, edges, features, classes, and the splits for training, validation, and testing, are summarized in Table~\ref{sup:data_split}.

\begin{table}[htbp]
\caption{Statistics of datasets used in the experiments.}
\label{sup:data_split}
\centering
\small
\begin{tabular}{lccccccc}
\toprule
\textbf{Datasets} & \textbf{\#Nodes} & \textbf{\#Edges} & \textbf{\#Features}  & \textbf{\#Classes}  & \textbf{\#Training}  & \textbf{\#Validation} & \textbf{\#Test}\\
\midrule
\textit{Polblogs}  & 1222  &  16714 & 1490 & 2 & 121 & 123 & 978\\
\textit{Cora}      & 2485  &  5069 & 1433 & 7 & 247 & 249 & 1988 \\
\textit{Citeseer}  & 2110  & 3668  & 3703 & 6 & 210 & 211 & 1688\\
\textit{Pubmed}    & 19717  & 44325  & 500 & 3 & 1971 & 1972 & 15774\\
\textit{Flickr}    & 7575 & 239738 & 12047 & 9 & 756 & 758 & 6060\\
\bottomrule
\end{tabular}
\end{table}

\section{Pseudo-Label Generation with K-Means}
\label{sec:pseudo}
When estimating \textbf{Graph Kernel Complexity (GKC)}, we replace true node labels with pseudo labels, making the Kernel Complexity-Based Edge Sanitization (KCES) method label-independent and solely data-driven. This design allows KCES to be broadly applicable to various model architectures without requiring any training. The specific procedure for generating pseudo labels is detailed in Algorithm~\ref{alg:kmeans}.

\begin{algorithm}[htbp] 
 \caption{Pseudo-Label Generation via K-Means Clustering} 
 \label{alg:kmeans} 
 \textbf{Input:} Feature matrix $\mathbb{X}$; adjacency matrix $\tilde{\mathbb{A}}$; number of clusters $K$ \\
 \textbf{Output:} Pseudo labels $\hat{\mathbf{y}}$
 \begin{algorithmic}[1] 
  \STATE $\mathbf{H} \leftarrow \tilde{\mathbb{A}} \mathbb{X}$ \hfill \small{\textit{// Aggregate node features via graph structure}} 
  \vspace{0.3em}
  \STATE $\mathbf{H}_{\text{norm}} \leftarrow \mathbf{H} / \|\mathbf{H}\|$ \hfill \small{\textit{// Normalize aggregated features row-wise}} 
  \vspace{0.3em}
  \STATE $\hat{\mathbf{y}} \leftarrow \text{KMeans}(\mathbf{H}_{\text{norm}}, K)$ \hfill \small{\textit{// Apply K-Means with $K$ clusters}} 
  \vspace{0.3em}
  \STATE \textbf{return} $\hat{\mathbf{y}}$
 \end{algorithmic} 
\end{algorithm}

Notably, we set the number of clusters $K$ equal to the number of classes in the original dataset.
\section{Experimental Setup}
\label{sec:exp-setup}

\subsection{Implementation Details}
\label{subsec:impl}

We adopt implementations of attack and defense baselines from the \textbf{DeepRobust} library~\cite{jin2020graph}, a comprehensive repository of adversarial attack and defense methods. Our full code is publicly available at: \url{https://anonymous.4open.science/r/KCScore-FB15/README.md}.

\subsection{Model Hyperparameters}
\label{subsec:hyper}
The hyper-parameters for our method across all datasets are listed in Table~\ref{tab:hyper}. 

\begin{table}[htbp]
    \centering
    \caption{Hyperparameters used in our experiments.}
    \resizebox{0.85\columnwidth}{!}{
    \begin{tabular}{lccccc}
    \toprule 
    & \textit{Polblogs} & \textit{Cora} & \textit{Citeseer} & \textit{Pubmed} & \textit{Flickr} \\
    \cmidrule(r){1-6} 
    \vspace{0.3em}
     \# Layers         & $2$ & $2$ & $2$ & $2$ & $2$ \\
     \vspace{0.3em}
     Hidden Dimension  & $[16, 2]$ & $[16, 7]$ & $[16, 6]$ & $[32, 3]$ & $[16, 9]$ \\
     \vspace{0.3em}
     Activation   & \multicolumn{5}{c}{ReLU used for all datasets} \\
      \vspace{0.3em}
     Dropout  & $0.05$ & $0.05$ & $0.05$ & $0.05$ & $0.05$ \\
      \vspace{0.3em}
     Optimizer         & \multicolumn{5}{c}{Adam with $1e$-$5$ weight decay} \\
      \vspace{0.3em}
     Learning Rate& $0.01$ & $0.01$ & $0.01$ & $0.01$ & $0.01$ \\
      \vspace{0.3em}
     Training Steps& $200$ & $200$ & $200$ & $200$ &  $200$ \\
     \bottomrule
    \end{tabular}}
    \label{tab:hyper}
\end{table}

\subsection{GPU Resources Used in Experiments}
\label{subsec:compute}
We perform all experiments using eight NVIDIA RTX A6000 GPUs (48GB memory each) on a dedicated server, enabling efficient training and evaluation across various datasets and attack settings.
\section{Additional Experimental Results}
\label{sec:supp-results}

\subsection{Edge Sanitization in Clean Settings}
\label{subsec:edge-clean}

In Section 6.3, we evaluate the impact of three pruning strategies: \textbf{\textit{High-KC Pruning}}, \textbf{\textit{Low-KC Pruning}}, and \textbf{\textit{Random Pruning}}, on GNN performance in adversarial settings. The results show that Low-KC edges are more beneficial to model performance, while High-KC edges tend to be harmful. To further support this finding, we extend our evaluation to clean graphs and conduct additional experiments on the \textit{Cora} and \textit{Pubmed} datasets.

\begin{figure}[htbp]\label{appendix:ES-Clean}
  \centering
  \includegraphics[width=1.0\linewidth]{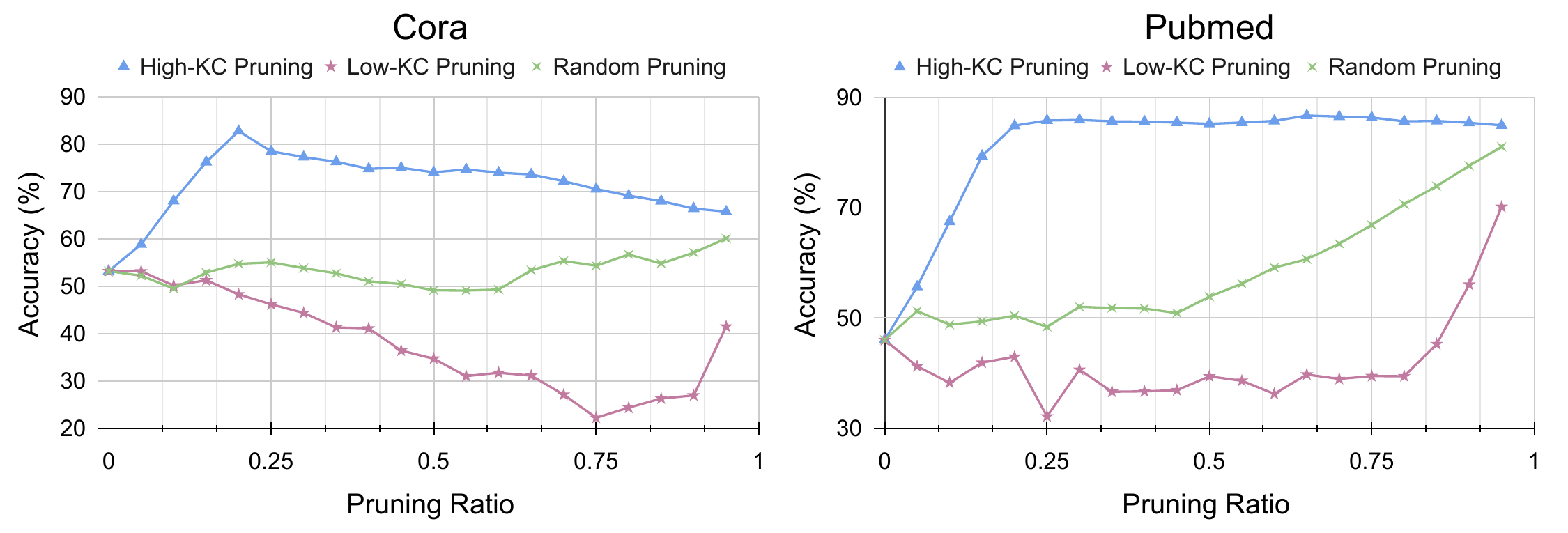}
  \caption{Comparing Edge Pruning Strategies via KC Scores in Adversarial Settings.}
  \label{fig-2:pruning-exp}
\end{figure}

The results in Figure~\ref{appendix:ES-Clean} further confirm that Low-KC edges benefit GNN performance, while High-KC edges are detrimental. Compared with \textbf{\textit{Random Pruning}}, which serves as a baseline to assess whether a pruning strategy is helpful or harmful, these findings demonstrate that \textbf{\textit{High-KC Pruning}} not only enhances generalization on clean graphs but also effectively improves GNN robustness against adversarial attacks.

\section{Theoretical Analysis of the Graph Kernel Model}
\label{sec:theory-gkm}

\subsection{Setup of Theory}
\label{subsec: theory-setup}
The theoretical analysis of our graph kernel model builds upon the framework established in~\cite{arora2019fine}. Before presenting our main results, we first summarize the key notations and settings used throughout this section for clarity. We study a two-layer Graph Neural Network (GNN) as a kernel model applied to an undirected graph with \(N\) nodes. The underlying graph distribution \(\mathcal{D}_G\), defined over \(\mathbb{R}^{N \times F} \times \mathbb{R}^N\), generates a graph \(G = (\mathbb{X}, \tilde{\mathbb{A}}, \tilde{\mathbb{D}}, \mathbf{y})\), where \((\mathbb{X}_i, y_i) \in \mathbb{R}^F \times \mathbb{R}\) are independently and identically distributed node feature-label pairs. The edge structure \(\tilde{\mathbb{A}} \in \{0, 1\}^{N \times N}\) is fixed, symmetric, and includes self-loops (\(\tilde{\mathbb{A}}_{ii} = 1\)), while the degree matrix is defined by \(\tilde{\mathbb{D}}_{ii} = \sum_j \tilde{\mathbb{A}}_{ij}\). We denote \(G_{\text{train}} = (\mathbb{X}, \tilde{\mathbb{A}}, \tilde{\mathbb{D}}, \mathbf{y})\) and \(G_{\text{test}} = (\mathbb{X}', \tilde{\mathbb{A}}, \tilde{\mathbb{D}}, \mathbf{y}')\) as two instantiations sampled from \(\mathcal{D}_G\), representing the training and test data, respectively.



The forward propagation of a two-layer GNN for node \(i\) is given by:
\begin{equation}
f_{\text{GNN}}(\mathbb{X}_i, \tilde{\mathbb{A}}, \tilde{\mathbb{D}}) = \frac{1}{\sqrt{m}} \sum_{r=1}^m a_r \, \sigma\left( W_r^{\top} \left( \tilde{\mathbb{D}}^{-\frac{1}{2}} \tilde{\mathbb{A}} \tilde{\mathbb{D}}^{-\frac{1}{2}} \mathbb{X} \right)_i \right),
\end{equation}
where \(m\) is the number of hidden units; \(\mathbb{W} \in \mathbb{R}^{F \times m}\) is the weight matrix of the first layer, with columns \(W_r \in \mathbb{R}^F\); \(\mathbf{a} = (a_1, \dots, a_m)^\top \in \mathbb{R}^m\) is the second-layer weight vector; and \(\sigma(\cdot)\) denotes the activation function (e.g., ReLU). The term \(\tilde{\mathbb{D}}^{-\frac{1}{2}} \tilde{\mathbb{A}} \tilde{\mathbb{D}}^{-\frac{1}{2}} \mathbb{X}\) represents the graph convolution operation, incorporating symmetric normalization and self-loops.

We define the training error (empirical risk) for a graph \(G_{\text{train}} = (\mathbb{X}, \tilde{\mathbb{A}}, \tilde{\mathbb{D}}, \mathbf{y})\) drawn from the distribution \(\mathcal{D}_G\) as:
\begin{equation}
L(\mathbb{W}) = \frac{1}{2} \sum_{i=1}^{N} \left( y_i - f_{\text{GNN}}(\mathbb{X}_i, \tilde{\mathbb{A}}, \tilde{\mathbb{D}})_i \right)^2,
\end{equation}
where \(f_{\text{GNN}}\) denotes the prediction of the two-layer GNN for node \(i\), and \(\mathbb{W}\) represents the learnable weight parameters.

The test error (expected risk) is defined as the expected empirical risk over the graph distribution \(\mathcal{D}_G\):
\begin{equation}
L_{\mathcal{D}_G}(\mathbb{W}) = \mathbb{E}_{G \sim \mathcal{D}_G} \left[ L(\mathbb{W}) \right].
\end{equation}

\paragraph{Assumptions.} 
We state the assumptions used in our theoretical analysis concerning the initialization, optimization, and distributional properties of the graph kernel model.

\begin{definition}[Non-degenerate Graph Distribution]
\label{non_de_graph}
A graph distribution \(\mathcal{D}_G\) over \(\mathbb{R}^{N \times F} \times \mathbb{R}^N\) is called \((\lambda_0, \delta, N)\)-non-degenerate if it generates graphs \(G = (\mathbb{X}, \tilde{\mathbb{A}}, \tilde{\mathbb{D}}, \mathbf{y})\) with \(N\) nodes, where each node feature-label pair \((\mathbb{X}_i, y_i)\) is independently drawn from the same distribution. The adjacency matrix \(\tilde{\mathbb{A}} \in \{0,1\}^{N \times N}\) is symmetric and includes self-loops (\(\tilde{\mathbb{A}}_{ii} = 1\)), and the degree matrix is defined as \(\tilde{\mathbb{D}}_{ii} = \sum_j \tilde{\mathbb{A}}_{ij}\). With probability at least \(1 - \delta\), the minimum eigenvalue of the Gram matrix \(\mathbb{H}^{\infty}\), constructed from the aggregated features \(\tilde{\mathbb{X}}_i = (\tilde{\mathbb{D}}^{-\frac{1}{2}} \tilde{\mathbb{A}} \tilde{\mathbb{D}}^{-\frac{1}{2}} \mathbb{X})_i\), satisfies \(\lambda_{\min}(\mathbb{H}^{\infty}) \geq \lambda_0 > 0\).
\end{definition}

\begin{assumption}[Graph Kernel Model Assumptions]
\label{assum}
The following assumptions are adopted throughout our analysis:
\begin{enumerate}[leftmargin=2.5em, itemsep=0.3em]
    \item \textbf{Full-Rank Adjacency Matrix.} The adjacency matrix \(\tilde{\mathbb{A}}\) (including self-loops) is full rank.
    
    \item \textbf{Normalized Aggregated Features.} The aggregated feature matrix is defined as \(\tilde{\mathbb{X}} = \tilde{\mathbb{D}}^{-\frac{1}{2}} \tilde{\mathbb{A}} \tilde{\mathbb{D}}^{-\frac{1}{2}} \mathbb{X}\). Without loss of generality, we assume that each row is normalized, i.e., \(\|\tilde{\mathbb{X}}_i\|_2 = 1\) for all \(i \in [N]\).
    
    \item \textbf{Regularity of Node Features.} With probability at least \(1 - \delta\), for any distinct \(i \neq j\), the feature vectors \(\mathbb{X}_i\) and \(\mathbb{X}_j\) are not parallel: \(\mathbb{X}_i \neq c \mathbb{X}_j\) for any scalar \(c \in \mathbb{R}\).
    
    \item \textbf{Bounded Labels.} For all nodes \((\mathbb{X}_i, y_i)\) sampled from \(\mathcal{D}_G\), the labels are bounded by \(|y_i| \leq 1\).
    
    \item \textbf{Gradient Descent Optimization.} The model parameters are updated using gradient descent:
    \begin{equation}
    \label{dynamic}
    \mathbb{W}_{t+1} = \mathbb{W}_t - \eta \nabla_{\mathbb{W}} L(\mathbb{W}_t; \mathbb{X}, \tilde{\mathbb{A}}, \tilde{\mathbb{D}}),
    \end{equation}
    where \(\eta\) is the learning rate.
    
    \item \textbf{Weight Initialization.} The first-layer weights \(\{W_r\}_{r=1}^m\) are independently initialized as \(W_r(0) \sim \mathcal{N}(\mathbf{0}, \kappa^2 \mathbb{I})\) for some \(\kappa \in (0, 1]\), and the second-layer coefficients \(\{a_r\}_{r=1}^m\) are fixed, drawn uniformly from \(\{-1, 1\}\).
    
    \item \textbf{Fixed Second Layer.} During training, only the first-layer weights \(\{W_r\}_{r=1}^m\) are updated. The second-layer weights \(\mathbf{a}\) remain fixed.
\end{enumerate}
\end{assumption}

\subsection{Theorem Introduction}
\subsubsection{Graph Kernel Complexity Bound}
\label{GKC_theorem}

Before presenting the formal theorems, we first state a lemma regarding the Gram Matrix \(\mathbb{H}^{\infty}\). This lemma shows that for a graph \(G\) drawn from the distribution \(\mathcal{D}_G\), both the Gram matrix \(\mathbb{H}^{\infty}\) of the original graph and the modified Gram matrix \(\mathbb{H}^{\infty}_{-(i,j)}\), obtained by removing any edge \(e(i, j)\), have their minimum eigenvalues bounded below by a positive constant \(\lambda_0 > 0\).

\begin{restatable}{lemma}{nondegenerate}
\label{lem:non-degenerate}
Under Assumption~\ref{assum}, with probability at least \(1 - \delta\), the random graph distribution \(\mathcal{D}_G\) is \((\lambda_0, \delta, N)\)-edge-robustly non-degenerate.
\end{restatable}

We now present the formal versions of Theorems~4.1 and 4.2, which provide theoretical guarantees on the behavior of the training and test errors based on properties of the Gram matrix.

\begin{theorem}
\label{train_theo_fo}
Under Assumption~\ref{assum}, let the initialization and optimization parameters satisfy $\kappa = O\left(\frac{\epsilon \delta}{\sqrt{N}}\right)$, $m = \Omega\left(\frac{N^7}{\lambda_0^4 \kappa^2 \delta^4 \epsilon^2}\right)$, and $\eta = O\left(\frac{\lambda_0}{N^2}\right)$. Then, with probability at least $1 - \delta$ over the random initialization, for all iterations $k = 0, 1, 2, \ldots$, the training error satisfies:
\begin{equation}
    L(\mathbb{W}_k; \mathbb{X}, \tilde{\mathbb{A}}, \tilde{\mathbb{D}}) = \sqrt{\sum_{i=1}^N \left(1 - \eta \lambda_i\right)^{2k} \left(\mathbf{v}_i^\top \mathbf{y}\right)^2} \pm \epsilon,
\end{equation}
where $\lambda_i$ and $\mathbf{v}_i$ denote the eigenvalues and corresponding eigenvectors of the Gram matrix $\mathbb{H}^{\infty}$.
\end{theorem}

\begin{theorem}
\label{test_theo_fo}
Fix $\delta \in (0,1)$. Under Assumption~\ref{assum}, let $\kappa = O\left(\frac{\lambda_0 \delta}{N}\right)$ and suppose $m \geq \kappa^{-2} \cdot \mathrm{poly}\left(N, \lambda_0^{-1}, \delta^{-1}\right)$. Then, for any number of gradient descent steps \(k \geq \Omega\left(\frac{1}{\eta \lambda_0} \log \frac{N}{\delta}\right)\), the test error of the two-layer GNN satisfies, with probability at least $1 - \delta$,
\begin{equation}
    L_{\mathcal{D}_G}(\mathbb{W}_k) \leq \sqrt{\frac{2 \mathbf{y}^{\top} \left(\mathbb{H}^{\infty}\right)^{-1} \mathbf{y}}{N}} + O\left(\sqrt{\frac{\log \frac{N}{\lambda_0 \delta}}{N}}\right).
\end{equation}
\end{theorem}

\subsubsection{Perturbed Graph KC-Score Bound}

\begin{corollary}
\label{edge_theo_fo}
Fix $\delta \in (0,1)$. Under Assumption~\ref{assum}, let $\kappa = O\left(\frac{\lambda_0 \delta}{N}\right)$ and suppose $m \geq \kappa^{-2} \cdot \mathrm{poly}\left(N, \lambda_0^{-1}, \delta^{-1}\right)$. For a two-layer GNN trained on a perturbed graph $G_{-(i,j)}$ (i.e., with edge $(i,j)$ removed) using gradient descent for $t \geq \Omega\left(\frac{1}{\eta \lambda_0} \log \frac{N}{\delta}\right)$ iterations, the following holds with probability at least $1 - \delta$:
\begin{equation}
    L_{\mathcal{D}_{G_{-(i,j)}}}(\mathbb{W}_t) \leq \sqrt{\mathrm{GKC}(\mathbb{H}^{\infty})} + \sqrt{KC(i,j)} + O\left(\sqrt{\frac{\log \frac{N}{\lambda_0 \delta}}{N}}\right),
\end{equation}
where $\mathrm{GKC}(\mathbb{H}^{\infty})$ denotes the Graph Kernel Complexity of the original graph, and $\mathrm{KC}(i,j)$ is the KC score associated with edge $(i,j)$, quantifying the change in GKC resulting from its removal.
\end{corollary}

\subsection{Background on Kernel Model Generalization}
We begin by reviewing the key theoretical results from~\cite{arora2019fine} concerning generalization in ReLU-based kernel models, specifically focusing on two-layer neural networks with ReLU activation, which serve as a foundation for our analysis.

\subsubsection{Setup from Prior Work}
The framework presented in~\cite{arora2019fine} analyzes the training dynamics and generalization properties of two-layer neural networks with ReLU activation. Below, we summarize the key elements of their setup to facilitate its extension to our graph-based context.

\paragraph{Model Setup} 
Consider a dataset of \(n\) input-label pairs \(\{(X_i, y_i)\}_{i=1}^n\), where each input \(X_i \in \mathbb{R}^d\) is a \(d\)-dimensional vector and each label \(y_i \in \mathbb{R}\) is a real-valued scalar. The pairs are independently drawn from a distribution \(\mathcal{D}\) over \(\mathbb{R}^d \times \mathbb{R}\). Let \(\mathbb{X} = (X_1, \ldots, X_n) \in \mathbb{R}^{d \times n}\) denote the input matrix and \(Y = (y_1, \ldots, y_n)^{\top} \in \mathbb{R}^n\) the label vector.

The two-layer neural network with \(m\) hidden units is defined as:
\begin{equation}
f_{\mathbb{W}, A}(X) = \frac{1}{\sqrt{m}} \sum_{r=1}^m a_r \sigma(W_r^{\top} X),
\end{equation}
where \(X \in \mathbb{R}^d\) is the input vector, \(W_r \in \mathbb{R}^d\) is the weight vector of the \(r\)-th neuron in the first layer, and \(a_r \in \mathbb{R}\) is the corresponding second-layer weight. The matrix \(\mathbb{W} = (W_1, \ldots, W_m) \in \mathbb{R}^{d \times m}\) stacks all first-layer weights, while \(A = (a_1, \ldots, a_m)^{\top} \in \mathbb{R}^m\) contains the second-layer weights. The activation function \(\sigma(z) = \max\{z, 0\}\) is the ReLU function. The scaling factor \(\frac{1}{\sqrt{m}}\) ensures that the output magnitude remains controlled as \(m\) increases.

\paragraph{Training Setup} To train the neural network, the work of~\cite{arora2019fine} applies gradient descent (GD) with random initialization to minimize the empirical loss, defined as the mean squared error over the training set:
\begin{equation}
L(\mathbb{W}; \mathbb{X}) = \frac{1}{2} \sum_{i=1}^n \left( y_i - f_{\mathbb{W}, A}(X_i) \right)^2,
\end{equation}
where \(f_{\mathbb{W}, A}(X_i)\) denotes the network's prediction for the input \(X_i\), and \(y_i\) is the associated label. The population loss, representing the expected test error, is given by:
\begin{equation}
L_{\mathcal{D}}(\mathbb{W}) = \mathbb{E}_{(\mathbb{X}, Y) \sim \mathcal{D}} \left[ L(\mathbb{W}; \mathbb{X}) \right].
\end{equation}

\paragraph{Assumption}
The analysis in~\cite{arora2019fine} is built upon a set of assumptions concerning the network initialization, optimization procedure, and the underlying data distribution. These assumptions establish the theoretical foundation required to analyze the generalization and convergence behavior of the two-layer ReLU network.

\begin{definition}
\label{def:non-degenerate}
A distribution \(\mathcal{D}\) over \(\mathbb{R}^d \times \mathbb{R}\) is said to be \((\lambda_0, \delta, n)\)-non-degenerate if, for \(n\) samples \(\{(X_i, y_i)\}_{i=1}^n\) independently drawn from \(\mathcal{D}\), the Gram matrix \(\mathbb{H}^{\infty} \in \mathbb{R}^{n \times n}\), whose entries are given by
\begin{equation}
\begin{aligned}
\mathbb{H}_{ij}^{\infty} 
&= \mathbb{E}_{W \sim \mathcal{N}(0, \mathbb{I})} 
\left[ X_i^{\top} X_j \cdot \mathbb{I}\{W^{\top} X_i \geq 0,\, W^{\top} X_j \geq 0\} \right] \\
&= \frac{X_i^{\top} X_j \left( \pi - \arccos(X_i^{\top} X_j) \right)}{2\pi},
\end{aligned}
\end{equation}
satisfies \(\lambda_{\min}(\mathbb{H}^{\infty}) \geq \lambda_0 > 0\) with probability at least \(1 - \delta\).
\end{definition}

\begin{assumption}
\label{assum:reference}
The following assumptions hold throughout the analysis in~\cite{arora2019fine}:
\begin{enumerate}[leftmargin=2.5em, itemsep=0.3em]
    \item \textbf{Normalized Inputs}: For all input-label pairs \((X_i, y_i)\) sampled from \(\mathcal{D}\), the inputs satisfy \(\|X_i\|_2 = 1\). This can be achieved by scaling each \(X_i\) by its \(\ell_2\)-norm.

    \item \textbf{Bounded Labels}: For all input-label pairs \((X_i, y_i)\) sampled from \(\mathcal{D}\), the labels satisfy \(|y_i| \leq 1\).

    \item \textbf{Data Distribution}: The data distribution \(\mathcal{D}\) is \((\lambda_0, \delta, n)\)-non-degenerate, as defined in Definition~\ref{def:non-degenerate}.

    \item \textbf{Parameter Initialization}: The first-layer weights \(\{W_r\}_{r=1}^m\) are initialized as \(W_r(0) \sim \mathcal{N}(0, \kappa^2 \mathbb{I})\) with \(0 < \kappa \leq 1\), and the second-layer coefficients \(\{a_r\}_{r=1}^m\) are sampled uniformly from \(\{-1, 1\}\). All initializations are mutually independent.

    \item \textbf{Fixed Second Layer}: During training, the second-layer coefficients \(A\) remain fixed, and only the first-layer weights \(\{W_r\}_{r=1}^m\) are optimized via gradient descent.

    \item \textbf{Gradient Descent}: The model is updated using gradient descent as follows:
    \begin{equation}
        W_r(k+1) = W_r(k) - \eta \frac{\partial L(\mathbb{W}(k); \mathbb{X})}{\partial W_r},
    \end{equation}
    where \(\eta > 0\) is the learning rate, and the gradient is given by:
    \begin{equation}
        \frac{\partial L(\mathbb{W}; \mathbb{X})}{\partial W_r} 
        = \frac{a_r}{\sqrt{m}} \sum_{i=1}^n \left( f_{\mathbb{W}, A}(X_i) - y_i \right) \mathbb{I}\{W_r^{\top} X_i \geq 0\} X_i.
    \end{equation}
\end{enumerate}
\end{assumption}

\subsubsection{Main Theorem of Previous Work}
We present two theorems from~\cite{arora2019fine} concerning the Gram matrix of a two-layer ReLU neural network, which characterize the behavior of the training error and test error, respectively.

\begin{lemma}
\label{ref_train_theo}
\textnormal{\cite{arora2019fine}}  
Under Assumption~\ref{assum:reference}, let \(\kappa = O\left(\frac{\epsilon \delta}{\sqrt{n}}\right)\), 
\(m = \Omega\left(\frac{n^7}{\lambda_0^4 \kappa^2 \delta^2 \epsilon^2}\right)\), and 
\(\eta = O\left(\frac{\lambda_0}{n^2}\right)\). Then, with probability at least \(1 - \delta\) over the random initialization, for all \(k = 0, 1, 2, \ldots\), the output of the neural network satisfies:
\begin{equation}
L(\mathbb{W}_k; \mathbb{X}) = \sqrt{\sum_{i=1}^n \left(1 - \eta \lambda_i\right)^{2k} \left(\mathbf{v}_i^{\top} \mathbf{y}\right)^2} \pm \epsilon,
\end{equation}
where \(\lambda_i\) are the eigenvalues of the Gram matrix \(\mathbb{H}^{\infty}\), and \(\mathbf{v}_i\) are the corresponding eigenvectors.
\end{lemma}

\begin{lemma}
\label{ref_test_theo}
\textnormal{\cite{arora2019fine}}  
Fix \(\delta \in (0,1)\). Under Assumption~\ref{assum:reference}, suppose \(\kappa = O\left(\frac{\lambda_0 \delta}{n}\right)\), and let \(m \geq \kappa^{-2} \cdot \mathrm{poly}(n, \lambda_0^{-1}, \delta^{-1})\). Then, for a two-layer neural network trained via gradient descent for at least \(k \geq \Omega\left(\frac{1}{\eta \lambda_0} \log \frac{n}{\delta} \right)\) iterations, the following bound on the test error holds with probability at least \(1 - \delta\):
\begin{equation}
L_{\mathcal{D}}(\mathbb{W}_k) \leq \sqrt{\frac{2 \mathbf{y}^{\top} \left(\mathbb{H}^{\infty}\right)^{-1} \mathbf{y}}{n}} + O\left( \sqrt{\frac{\log \frac{n}{\lambda_0 \delta}}{n}} \right).
\end{equation}
\end{lemma}

\subsection{Proof of Main Theorem}
\subsubsection{Lemmas Used in the Proof}
\begin{lemma}
\label{lem:non-parallel}
Let \(\mathbb{X} \in \mathbb{R}^{N \times F}\) have non-parallel rows \(\mathbb{X}_i \neq c \mathbb{X}_j\) for \(i \neq j\), \(c \in \mathbb{R}\). Given full-rank \(\tilde{\mathbb{A}} = \mathbb{A} + \mathbb{I} \in \mathbb{R}^{N \times N}\) and \(\tilde{\mathbb{D}} \in \mathbb{R}^{N \times N}\), the aggregated features \(\tilde{\mathbb{X}} = \tilde{\mathbb{D}}^{-\frac{1}{2}} \tilde{\mathbb{A}} \tilde{\mathbb{D}}^{-\frac{1}{2}} \mathbb{X} \in \mathbb{R}^{N \times F}\) satisfy \(\tilde{\mathbb{X}}_i \neq c \tilde{\mathbb{X}}_j\) for \(i \neq j\), \(c \in \mathbb{R}\).
\end{lemma}

\begin{proof}
Define the symmetric aggregation matrix
\begin{equation}
\mathbb{T} \;=\; \tilde{\mathbb{D}}^{-\frac12}\,\tilde{\mathbb{A}}\,\tilde{\mathbb{D}}^{-\frac12}\,.
\end{equation}
Since \(\tilde{\mathbb{D}}\) is diagonal with strictly positive entries and \(\tilde{\mathbb{A}}\) is full–rank, \(\mathbb{T}\) is invertible. Hence,
\begin{equation}
\tilde{\mathbb{X}} = \mathbb{T}\,\mathbb{X},
\quad
\tilde{\mathbb{X}}_i = \mathbb{T}_i\,\mathbb{X},
\end{equation}
where \(\mathbb{T}_i\) denotes the \(i\)-th row of \(\mathbb{T}\).

Assume, for the sake of contradiction, that there exist \(i\neq j\) and \(c\in\mathbb{R}\) such that
\begin{equation}
\tilde{\mathbb{X}}_i = c\,\tilde{\mathbb{X}}_j.
\end{equation}
Then
\begin{equation}
\mathbb{T}_i\,\mathbb{X} = c\,\mathbb{T}_j\,\mathbb{X}.
\end{equation}
Writing \(\mathbb{X}=[\mathbf{x}_1,\dots,\mathbf{x}_F]\), we have for each \(f\in[F]\):
\begin{equation}
(\mathbb{T}_i - c\,\mathbb{T}_j)\,\mathbf{x}_f = 0.
\end{equation}
Because \(\mathbb{T}\) is invertible, \(\mathbb{T}_i\) and \(\mathbb{T}_j\) are linearly independent, so \(\mathbb{T}_i - c\,\mathbb{T}_j \neq 0\). Meanwhile, the set \(\{\mathbf{x}_f\}\) spans a subspace of dimension at least two by the non-parallelism of the original rows of \(\mathbb{X}\). This contradicts the assumption that a nonzero row vector annihilates every \(\mathbf{x}_f\). Therefore no such \(c\) exists, and the rows of \(\tilde{\mathbb{X}}\) remain pairwise non-parallel.
\end{proof}

\begin{lemma}[Theorem 3.1 of \cite{du2018gradient}]
\label{theo:non-parallel}
Assume the node feature matrix \(\mathbb{X}\in\mathbb{R}^{N\times F}\) has pairwise non-parallel rows, i.e.\ for all \(i\neq j\) and any scalar \(c\), \(\mathbb{X}_i\neq c\,\mathbb{X}_j\).  Then the Gram matrix \(\mathbb{H}^{\infty}\) satisfies
\begin{equation}
\lambda_0 = \lambda_{\min}(\mathbb{H}^{\infty}) > 0.
\end{equation}
\end{lemma}

\begin{proof}
By Assumption~\ref{assum}, with probability at least \(1-\delta\) the raw features satisfy 
\begin{equation}
\mathbb{X}_i \neq c\,\mathbb{X}_j
\quad\forall\,i\neq j,\;c\in\mathbb{R}.
\end{equation}
For the original graph, let
\begin{equation}
\mathbb{T} \;=\; \tilde{\mathbb{D}}^{-\frac12}\,\tilde{\mathbb{A}}\,\tilde{\mathbb{D}}^{-\frac12},
\quad
\tilde{\mathbb{X}} = \mathbb{T}\,\mathbb{X}.
\end{equation}
Since \(\tilde{\mathbb{D}}\) and \(\tilde{\mathbb{A}}\) are full rank, \(\mathbb{T}\) is invertible.  Lemma~\ref{lem:non-parallel} then guarantees
\begin{equation}
\tilde{\mathbb{X}}_i \neq c\,\tilde{\mathbb{X}}_j
\quad\forall\,i\neq j,\;c\in\mathbb{R}.
\end{equation}
Applying Lemma~\ref{theo:non-parallel} to \(\tilde{\mathbb{X}}\) yields
\begin{equation}
\lambda_{\min}\bigl(\mathbb{H}^{\infty}\bigr) \;\ge\;\lambda_1 \;>\;0,
\end{equation}
where 
\begin{equation}
\mathbb{H}_{ij}^{\infty}
= \mathbb{E}_{W\sim\mathcal{N}(0,I)}
\bigl[\tilde{\mathbb{X}}_i^{\!\top}\tilde{\mathbb{X}}_j\,
\mathbb{I}\{W^{\top}\tilde{\mathbb{X}}_i\!\ge0,\,
W^{\top}\tilde{\mathbb{X}}_j\!\ge0\}\bigr].
\end{equation}

Next, remove edge \((i,j)\) to obtain
\begin{equation}
\tilde{\mathbb{A}}_{-(i,j)}
= \tilde{\mathbb{A}} - \mathbf{e}_i\mathbf{e}_j^{\!\top} - \mathbf{e}_j\mathbf{e}_i^{\!\top},
\quad
\tilde{\mathbb{D}}_{-(i,j)}
= \mathrm{diag}\bigl(\tilde{\mathbb{A}}_{-(i,j)}\,\mathbf{1}\bigr),
\end{equation}
and define
\begin{equation}
\mathbb{T}_{-(i,j)}
= \tilde{\mathbb{D}}_{-(i,j)}^{-\frac12}\,
  \tilde{\mathbb{A}}_{-(i,j)}\,
  \tilde{\mathbb{D}}_{-(i,j)}^{-\frac12},
\quad
\tilde{\mathbb{X}}_{-(i,j)}
= \mathbb{T}_{-(i,j)}\,\mathbb{X}.
\end{equation}
By the same rank argument, \(\mathbb{T}_{-(i,j)}\) is invertible, so Lemma~\ref{lem:non-parallel} and Lemma~\ref{theo:non-parallel} together imply
\begin{equation}
\lambda_{\min}\bigl(\mathbb{H}^{\infty}_{-(i,j)}\bigr)
\;\ge\;\lambda_{-(i,j)} \;>\;0.
\end{equation}

Finally, set 
\begin{equation}
\lambda_0 \;=\;\min\{\lambda_1,\;\lambda_{-(i,j)}:\;(i,j)\in E\},
\end{equation}
which is strictly positive.  Hence with probability at least \(1-\delta\) both
\(\lambda_{\min}(\mathbb{H}^{\infty})\ge\lambda_0\) and
\(\lambda_{\min}(\mathbb{H}^{\infty}_{-(i,j)})\ge\lambda_0\),
satisfying Definition~\ref{def:non-degenerate}.
\end{proof}


\begin{lemma}
\label{lem:iid-aggregated}
Suppose \(\{(\mathbb{X}_i,y_i)\}_{i=1}^N\) are independent draws from a distribution \(\mathcal{D}_G\) over \(\mathbb{R}^F\times\mathbb{R}\).  Let \(\tilde{\mathbb{A}}\in\{0,1\}^{N\times N}\) be a fixed symmetric adjacency matrix with self‐loops (\(\tilde{\mathbb{A}}_{ii}=1\)) and let \(\tilde{\mathbb{D}}\) be its degree matrix defined by \(\tilde{\mathbb{D}}_{ii}=\sum_j\tilde{\mathbb{A}}_{ij}\).Define
\begin{equation}
\tilde{\mathbb{X}}_i \;=\;\bigl(\tilde{\mathbb{D}}^{-\tfrac12}\,\tilde{\mathbb{A}}\,\tilde{\mathbb{D}}^{-\tfrac12}\,\mathbb{X}\bigr)_i.
\end{equation}
Then \(\{(\tilde{\mathbb{X}}_i,y_i)\}_{i=1}^N\) are also independent samples from \(\mathcal{D}_G\).
\end{lemma}

\begin{proof}
Let $\mathbf{A} \;=\;\tilde{\mathbb{D}}^{-1/2}\,\tilde{\mathbb{A}}\,\tilde{\mathbb{D}}^{-1/2}$ and write $\mathbb{X}=(\mathbb{X}_1,\dots,\mathbb{X}_N)$.  Then for each $i$,
\begin{equation}
\tilde{\mathbb{X}}_i \;=\;(\mathbf{A}\,\mathbb{X})_i
\;=\;\sum_{j=1}^N A_{ij}\,\mathbb{X}_j.
\end{equation}
Since the pairs $(\mathbb{X}_j,y_j)$ are i.i.d., each $\mathbb{X}_j$ has the same law and is independent of the fixed coefficients $A_{ij}$.  Hence the $\tilde{\mathbb{X}}_i$ are identically distributed, and the labels $y_i$ are likewise i.i.d., so the pairs $(\tilde{\mathbb{X}}_i,y_i)$ are identically distributed.

Next, for $i\neq k$,
\begin{equation}
\tilde{\mathbb{X}}_i = \sum_{j=1}^N A_{ij}\,\mathbb{X}_j,
\quad
\tilde{\mathbb{X}}_k = \sum_{j=1}^N A_{kj}\,\mathbb{X}_j.
\end{equation}
Because the $\mathbb{X}_j$ are mutually independent and the $A_{ij}$ are constants, it follows that $\tilde{\mathbb{X}}_i\perp \tilde{\mathbb{X}}_k$.  Also, $y_i$ and $y_k$ are independent of each other and of all the $\mathbb{X}_j$.  Therefore
\begin{equation}
(\tilde{\mathbb{X}}_i,y_i)\;\perp\;(\tilde{\mathbb{X}}_k,y_k)
\quad\text{for }i\neq k.
\end{equation}
Combining identical distribution with mutual independence shows that $\{(\tilde{\mathbb{X}}_i,y_i)\}_{i=1}^N$ are i.i.d.
\end{proof}

\subsubsection{Proof of Theorem \ref{train_theo_fo} and Theorem \ref{test_theo_fo}}

\begin{proof}[Proof of Theorem E.4]
Define $\mathbf{A} = \tilde{\mathbb{D}}^{-1/2}\,\tilde{\mathbb{A}}\,\tilde{\mathbb{D}}^{-1/2}$ and $\tilde{\mathbb{X}} = \mathbf{A}\,\mathbb{X}$.  Moreover,
\begin{equation}\label{eq:fGNN-E4}
f_{\mathrm{GNN}}(\mathbb{X},\tilde{\mathbb{A}},\tilde{\mathbb{D}})_i
= \frac{1}{\sqrt m}\sum_{r=1}^m a_r\,\sigma\bigl(W_r^\top\tilde{\mathbb{X}}_i\bigr),
\end{equation}
and the training loss is
\begin{equation}\label{eq:loss-E4}
L(\mathbb{W}_k;\mathbb{X},\tilde{\mathbb{A}},\tilde{\mathbb{D}})
=\tfrac12\sum_{i=1}^N\bigl(y_i - f_{\mathrm{GNN}}(\mathbb{X},\tilde{\mathbb{A}},\tilde{\mathbb{D}})_i\bigr)^2.
\end{equation}

By Lemma \ref{lem:iid-aggregated}, the pairs $(\tilde{\mathbb{X}}_i,y_i)$ are independent with identical marginals.  Relabeling $\tilde{\mathbb{X}}_i\mapsto X_i$, $\mathcal{D}_G\mapsto\mathcal{D}$ and $N=n$, Lemma \ref{lem:non-degenerate} gives
\begin{equation}\label{eq:Hmin-E4}
\lambda_{\min}(\mathbb{H}^\infty)\;\ge\;\lambda_0>0,
\end{equation}
and Assumption \ref{assum} (unit‐norm, edge‐robust non‐degeneracy) matches Assumption \ref{assum:reference}.

Choose
\begin{equation}\label{eq:params-E4}
\kappa = O\Bigl(\frac{\epsilon\,\delta}{\sqrt N}\Bigr),\quad
m = \Omega\Bigl(\frac{N^7}{\lambda_0^4\,\kappa^2\,\delta^4\,\epsilon^2}\Bigr),\quad
\eta = O\Bigl(\frac{\lambda_0}{N^2}\Bigr).
\end{equation}
Then by Lemma \ref{ref_train_theo}, with probability at least $1-\delta$,
\begin{equation}\label{eq:train-E4}
L(\mathbb{W}_k;\mathbb{X},\tilde{\mathbb{A}},\tilde{\mathbb{D}})
= \sqrt{\sum_{i=1}^N (1-\eta\lambda_i)^{2k}\,(\mathbf{v}_i^\top\mathbf{y})^2}\;\pm\;\epsilon.
\end{equation}
\end{proof}

\begin{proof}[Proof of Theorem E.5]
Fix $\delta\in(0,1)$.  Under Assumption E.2 let
\[
\kappa = O\Bigl(\frac{\lambda_0\,\delta}{N}\Bigr),
\qquad
m \ge \kappa^{-2}\,\mathrm{poly}(N,\lambda_0^{-1},\delta^{-1}),
\]
and take any
\[
k \ge \Omega\!\Bigl(\tfrac{1}{\eta\lambda_0}\log\tfrac{N}{\delta}\Bigr).
\]
By Lemma \ref{ref_test_theo}, with probability at least $1-\delta$, we obtain the bound
\begin{equation}\label{eq:test-E5}
L_{\mathcal{D}}(\mathbb{W}_k)
\le \sqrt{\frac{2\,\mathbf{y}^\top(\mathbb{H}^\infty)^{-1}\mathbf{y}}{n}}
\;+\;
O\!\Bigl(\sqrt{\tfrac{\log\!\bigl(N/(\lambda_0\delta)\bigr)}{n}}\Bigr).
\end{equation}
\end{proof}

\subsubsection{Proof of Corollary \ref{edge_theo_fo}}
\begin{proof}
By Assumption \ref{assum} and Lemma \ref{lem:non-degenerate}, with probability at least $1-\delta$,
\begin{equation}\label{eq:Hmin-minus}
\lambda_{\min}(\mathbb{H}^\infty)\;\ge\;\lambda_0
\quad\text{and}\quad
\lambda_{\min}\bigl(\mathbb{H}^\infty_{-(i,j)}\bigr)\;\ge\;\lambda_0.
\end{equation}
Applying Theorem \ref{test_theo_fo} to the subgraph $G_{-(i,j)}$ (whose Gram matrix is $\mathbb{H}^\infty_{-(i,j)}$), under
\[
\kappa = O\!\Bigl(\tfrac{\lambda_0\,\delta}{N}\Bigr), 
\quad
m \ge \kappa^{-2}\,\mathrm{poly}(N,\lambda_0^{-1},\delta^{-1}), 
\quad
t \ge \Omega\!\Bigl(\tfrac{1}{\eta\,\lambda_0}\log\tfrac{N}{\delta}\Bigr),
\]
we get with probability at least $1-\delta$,
\begin{equation}\label{eq:test-minus}
L_{\mathcal{D}_{G_{-(i,j)}}}(\mathbb{W}_t)
\le
\sqrt{\mathrm{GKC}\bigl(\mathbb{H}^\infty_{-(i,j)}\bigr)}
\;+\;
O\!\Bigl(\sqrt{\tfrac{\log\bigl(N/(\lambda_0\delta)\bigr)}{N}}\Bigr),
\end{equation}
where
\[
\mathrm{GKC}\bigl(\mathbb{H}^\infty_{-(i,j)}\bigr)
=\frac{2\,\mathbf{y}^\top\bigl(\mathbb{H}^\infty_{-(i,j)}\bigr)^{-1}\mathbf{y}}{N}.
\]
Set
\[
a = \mathrm{GKC}\bigl(\mathbb{H}^\infty_{-(i,j)}\bigr), 
\quad
b = \mathrm{GKC}(\mathbb{H}^\infty)
=\frac{2\,\mathbf{y}^\top(\mathbb{H}^\infty)^{-1}\mathbf{y}}{N},
\quad
c = \mathrm{KC}(i,j)
=\lvert a-b\rvert.
\]
Since $a,b,c\ge0$ and $\sqrt{a}\le\sqrt{b}+\sqrt{c}$, it follows that with probability at least $1-\delta$,
\begin{equation}\label{eq:final-minus}
L_{\mathcal{D}_{G_{-(i,j)}}}(\mathbb{W}_t)
\le
\sqrt{\mathrm{GKC}(\mathbb{H}^\infty)}
\;+\;
\sqrt{\mathrm{KC}(i,j)}
\;+\;
O\!\Bigl(\sqrt{\tfrac{\log\bigl(N/(\lambda_0\delta)\bigr)}{N}}\Bigr).
\end{equation}
\end{proof}

\section{Social Impact}
This work contributes to improving the robustness and reliability of Graph Neural Networks (GNNs), which are widely used in sensitive applications such as fraud detection, drug discovery, recommendation systems, and social network analysis. By introducing a principled, training-free defense framework—Kernel Complexity-based Edge Sanitization (KCES)—that removes structurally harmful edges using a novel Graph Kernel Complexity (GKC) measure, the proposed method enhances generalization without relying on model-specific assumptions. Its architecture-agnostic and data-driven nature makes it broadly applicable with minimal overhead, potentially increasing trust in GNN systems deployed in real-world, high-stakes environments. Additionally, by offering theoretical insight into which graph structures support or hinder generalization, this work may also aid in interpretability and scientific analysis.

However, the ability to selectively prune influential edges could be misused, for example, to censor specific subgraphs or manipulate network structures for malicious purposes, particularly in social or political domains. Moreover, KCES assumes the underlying graph data is reliable; in practice, noise or bias in the input graph could lead to the removal of valuable or fair connections, thereby degrading model performance or fairness. Lastly, while KCES improves robustness under certain adversarial conditions, it does not offer complete security and may induce a false sense of safety if used without complementary safeguards. Ethical considerations should guide the deployment of this technology, especially in sensitive or dynamic environments.

\end{document}